\documentclass{article}
\usepackage[utf8]{inputenc}
\usepackage{enumerate}
\usepackage{amsfonts,amsmath,amssymb,amsthm,bbm,bm}
\usepackage{mathrsfs}
\usepackage{mathtools}
\usepackage{xcolor}
\usepackage{soul}
\setstcolor{red}

\usepackage{longtable}

\usepackage{graphicx}
\usepackage{caption}
\usepackage{subcaption}
\usepackage{tabularx, booktabs}

\usepackage{float}

\usepackage[style=alphabetic,backend=biber, sorting = nyt, maxbibnames=99]{biblatex}
\usepackage{hyperref}
\usepackage{cleveref} 

\newcommand\calD{{\mathcal{D}}}

\newcommand\calF{{\mathcal{F}}}

\newcommand\calH{{\mathcal{H}}}
\newcommand\calI{{\mathcal{I}}}

\newcommand\calK{{\mathcal{K}}}
\newcommand\calL{{\mathcal{L}}}

\newcommand\calN{{\mathcal{N}}}

\newcommand\calP{{\mathcal{P}}}
\newcommand\calQ{{\mathcal{Q}}}
\newcommand\calR{{\mathcal{R}}}

\newcommand\calT{{\mathcal{T}}}

\newcommand\calV{{\mathcal{V}}}

\newcommand{\Bcal}{\mathcal{B}}

\newcommand{\Hcal}{\mathcal{H}}

\newcommand{\Kcal}{\mathcal{K}}

\newcommand{\Pcal}{\mathcal{P}}

\newcommand{\Wcal}{\mathcal{W}}

\newcommand\sH{{\mathbb{H}}}

\newcommand\sP{{\mathbb{P}}}

\newcommand{\NN}{\mathbb{N}}

\newcommand{\RR}{\mathbb{R}}

\newcommand{\W}{\mathcal{W}}

\newcommand{\CW}{\mathcal{CW}}

\newcommand{\AW}{\mathcal{AW}}

\newcommand{\DKL}{\calD_{\mathrm{KL}}}

\newcommand{\TVD}{\mathrm{TV}}
\newcommand{\AVD}{\mathrm{AV}}

\newcommand{\cpl}{\Pi}
\newcommand{\bccpl}{\Pi_{\text{bc}}}
\newcommand{\ccpl}{\Pi_{\text{c}}}

\newcommand{\E}{\mathbb{E}}
\newcommand{\R}{\mathbb{R}}
\newcommand{\N}{\mathbb{N}}

\newcommand{\avar}{\mathrm{AVaR}}
\newcommand{\var}{\mathrm{VaR}}
\newcommand{\es}{\mathrm{ES}}

\newcommand{\defeq}{\vcentcolon=}

\newcommand{\mudata}{\mu_{\mathrm{data}}}

\newcommand{\murecon}{\mu_{\mathrm{rec}}}

\newcommand{\murec}{\mu_{\mathrm{rec}}}

\newcommand{\mugen}{\mu_{\mathrm{gen}}}
\newcommand{\mulatent}{\mu_{\mathrm{latent}}}
\newcommand{\muprior}{\mu_{\mathrm{prior}}}
\newcommand{\muemp}{\hat{\mu}_{\mathrm{data}}}

\newcommand{\platent}{p_{\mathrm{latent}}}
\newcommand{\pprior}{p_{\mathrm{prior}}}

\newcommand{\De}{\mathrm{De}}

\newcommand{\Lrecon}{\mathcal{L}_{\mathrm{rec}}}
\newcommand{\Llatent}{\mathcal{L}_{\mathrm{latent}}}

\usepackage{thmtools}
\usepackage{hyperref}
\usepackage{cleveref}
\newtheorem{theorem}{Theorem}[section]
\newtheorem{corollary}[theorem]{Corollary}

\newtheorem{lemma}[theorem]{Lemma}

\theoremstyle{remark}
\newtheorem{remark}[theorem]{Remark}

\theoremstyle{definition}
\newtheorem{definition}[theorem]{Definition}
\newtheorem{example}[theorem]{Example}

\crefname{section}{Section}{Section}

\numberwithin{equation}{section} 

\definecolor{newblue}{rgb}{0.0, 0.0, 0.9}

\title{Time-Causal VAE: Robust Financial Time Series Generator}

\author{Beatrice Acciaio\thanks{Department of Mathematics, ETH Z\"{u}rich, Switzerland.\newline \emph{beatrice.acciaio@math.ethz.ch}, \emph{songyan.hou@math.ethz.ch}}, \,Stephan Eckstein\thanks{Department of Mathematics, University of T\"{u}bingen, Germany. \newline \emph{stephan.eckstein@uni-tuebingen.de}} \,and\, Songyan Hou\footnotemark[1]}
\date{\today }

\begin{document}

\maketitle

\begin{abstract}
We build a time-causal variational autoencoder (TC-VAE) for robust generation of financial time series data. Our approach imposes a \emph{causality constraint} on the encoder and decoder networks, ensuring a causal transport from the real market time series to the fake generated time series. Specifically, we prove that the TC-VAE loss provides an upper bound on the causal Wasserstein distance between market distributions and generated distributions. Consequently, the TC-VAE loss controls the discrepancy between optimal values of various dynamic stochastic optimization problems under real and generated distributions. To further enhance the model's ability to approximate the latent representation of the real market distribution, we integrate a RealNVP prior into the TC-VAE framework. Finally, extensive numerical experiments show that TC-VAE achieves promising results
on both synthetic and real market data. This is done by comparing real and generated distributions according to various statistical distances, demonstrating the effectiveness of the generated data for downstream financial optimization tasks, as well as showcasing that the generated data reproduces stylized facts of real financial market data.

\noindent\emph{Keywords:} adapted Wasserstein distance, empirical measure, convergence rate, kernel smoothing\\
MSC (2020): 37M10, 68T07

\end{abstract}

\section{Introduction}
\label{sect.intro}
For financial time series, the shortage of samples makes it statistically hard for empirical processes to achieve an acceptable confidence level in describing the underlying market distribution. In practice, it is widely recognized among financial engineers that back-testing exclusively on empirical market data results in significant over-fitting, which leads to unpredictably high risks in decision making based on these tests \cite{Bailey2016Tpo}. Synthetic data are therefore generated to augment scarce market data, and used to improve  back-testing, stress-testing, exploring new scenarios, and in deep learning processes in financial applications; see the overview given in \cite{Assefa2020Gsd}. For those purposes, the generated data should look like plausible samples from
the underlying market distribution, for example reproducing stylized facts observed in the market. In particular, we want the distribution of the generated data to be close to the underlying market distribution in their performance on decision making problems, such as pricing and hedging, as well as optimal stopping and utility maximization. Notably, these problems are not continuous with respect to widely used distances, such as the Maximum Mean Discrepancy (MMD) and the Wasserstein distances ($\W$-distances). On the other hand, these problems are Lipschitz-continuous with respect to stronger metrics, called adapted Wasserstein distances ($\AW$-distances) \cite{BackhoffVeraguas2020AWd, Pflug2014Mso}. Therefore, for the augmentation of market data with the purpose of e.g.~testing performance of different strategies within any of the above problems,
it is desirable to find a generated distribution which is close to the underlying market distribution in $\AW$-distance. 

In that spirit, one natural choice would be to build a generative adversarial network (GAN) employing adapted distances as loss functions. Xu et al.~\cite{Xu2020CgG} introduce the COT-GAN, using causal Wasserstein distances ($\CW$-distances) as a compromise between $\W$-distances and $\AW$-distances. Since $\CW$-distances are still considerably more expensive to compute than $\W$-distances in a multi-step setting, COT-GAN can only provide satisfactory results for time series with few time steps. This limits the application of $\CW$-GAN for financial time series generation. Aside from the time-step constraint, recent research has shown that distributions generated via Wasserstein GANs often are far from the source distributions in  $\W$-distance \cite{Stanczuk2021WGw}. This also explains why GANs show mode collapse \cite{thanh2020catastrophic}, which refers to a scenario where the generator starts producing a limited variety of outputs, often very similar to each other, instead of a diverse range that represents the real data distribution.
Not to mention that the adversarial training to find the saddle point of the min-max problem is notoriously unstable. Lastly, GANs are also usually ``data hungry" \cite{karras2020training}, and scarcity of market data is the initial problem we started with.

For these reasons, we decided to avoid adversarial minimization, and instead adopt the network structure of variational autoencoders (VAEs) introduced in \cite{Kingma2014Aev}. VAEs are highly expressive models that retain the computational efficiency of fully factorized models \cite{cinelli2021variational} and have found wide applications in generating data for speech, images, and text \cite{bowman2015generating}.
Notably, very deep VAEs generalize autoregressive models and can outperform them on images \cite{child2020very}. Moreover, VAEs frameworks are not only useful in generation, but also able to learn a disentangled latent representation of the data distribution, see \cite{chen2018isolating, mathieu2019disentangling}. This is in particular true for $\beta$-VAE, which we use in the present paper, see \cite{burgess2018understanding, bhowalvariational}. Recently, a series of papers have presented different extensions of 
VAEs to process sequential data, see for example the summary paper \cite{girin2020dynamical}. In the present paper, we introduce a variation of VAEs, which is able to learn the conditional distribution of financial time series under $\CW$-distance. Specifically, the encoder maps the market underlying data distribution $\mudata$ into a latent distribution $\mulatent$ on the latent space, while the decoder maps the latent distribution back to a reconstructed distribution $\murecon$ on the data space. The decoder will be used to generate a distribution $\mugen$ by pushing a prior distribution $\muprior$ defined on the latent space. 
As it is common for VAEs, we want to achieve two goals at the same time: 1) minimize the reconstruction error $\Lrecon$ between $\mudata$ and $\murecon$; 2) minimize the latent error $\Llatent$ between $\mulatent$ and the prior distribution $\muprior$. As a result, the generated distribution $\mugen$ should also be close to the data distribution $\mudata$.

Crucially, we incorporate two modifications to VAEs:
\begin{enumerate}[(i)]
    \item \textbf{Causality constraint:} we impose a causality condition on the encoder and decoder, so that the reconstruction path at time $t$ depends on the input path only up to time $t$. We name the resulting network structure \emph{Time-Causal}-VAE (TC-VAE).
    \item \textbf{Flexible prior:} we apply a flexible learnable prior distribution $\muprior$, and specifically the RealNVP introduced in \cite{Dinh2016Deu, gatopoulos2020super}. 
\end{enumerate}
Networks with a time-causal structure are already present in the literature. Those include, for example, recurrent neural networks and causal self-attention networks, both proven highly successful in time series generation; see \cite{chung2015recurrent,Yoon2019TsG,Esteban2017RvM, yang2021causal}. 

From the causal optimal transport point of view, our encoder and decoder together transport the market data distribution to the reconstructed distribution in a causal fashion. Consequently, we can prove that  the $\CW$-distance between $\mudata$ and $\murecon$ is bounded by the reconstruction error $\Lrecon$. On the other hand, RealNVP has been proven very successful in approximating distributions, and its density computationally very tractable \cite{dinh2016density}, which allows an easy computation of the KL-divergence. The flexibility and tractability of RealNVP empowers TC-VAE such that $\muprior$ and $\mulatent$ are close enough {(in KL-divergence)} to control the $\CW$-distance between $\mugen$ and $\murec$. Consequently, the TC-VAE loss controls the 
$\CW$-distance between $\mugen$ and $\mudata$,  thereby providing one-sided guarantees of such control problems through the $\CW$-distance.

With these improvements, TC-VAE achieves the goals which we laid out above, generating financial time series data with strong statistical guarantees according to causal Wasserstein distances and showcasing promising numerical results for financial tasks. On synthetic datasets like the Black-Scholes model, Heston model and Path-Dependent-Volatility model, TC-VAE  learns the data distribution very well in terms of drift, volatility, marginal distribution, Wasserstein distance \cite{Arjovsky2017WGA}, Gaussian maximum mean discrepancies \cite{Gretton2012Akt}, Signature maximum mean discrepancies \cite{liao2024sig}, adapted Wasserstein distance \cite{BackhoffVeraguas2020AWd}, and optimal values of multistage optimization problems, like mean-variance portfolio optimization \cite{Forsyth2022MPM}, log-utility maximization \cite{Merton1975Oca}, and optimal stopping \cite{Becker2019Dos}. On real market datasets, such as S\&P~500 and VIX, conditional TC-VAE enables us to generate paths, as many as possible and as long as possible. The generated paths reproduce stylized facts of financial time series \cite{Cont2001Epo} capturing key properties such as gain/loss asymmetry, skewness and kurtosis of returns, heavy-tail returns, no correlation in returns, short time correlation in square returns, long time correlation in absolute returns, and volatility clustering.\\

\noindent \textbf{Organization of the paper.} In the rest of \cref{sect.intro}, we give a brief overview of related works and introduce relevant notation. In \cref{sect.cau_gen}, we introduce the architecture of TC-VAE and its loss function. In \cref{sect.rob} we prove robustness of stochastic optimization problems w.r.t.~causal distances. Finally, in \cref{sec.exp}  we show that TC-VAE achieves promising results on both unconditional and conditional financial time series generation on several different metrics and datasets\footnote{The code and data are available at \href{https://github.com/justinhou95/TimeCausalVAE}{https://github.com/justinhou95/TimeCausalVAE}.}.

\paragraph{Related Literature.}
Numerous methodologies have been explored for generating financial time series. \cite{Kondratyev2019TMG} were among the first to use restricted Boltzmann machines to generate synthetic foreign exchange rates. Recent advances in deep learning have introduced promising techniques, including Variational Autoencoders (VAEs) introduced in \cite{Kingma2014Aev} and Generative Adversarial Networks (GANs) pioneered in \cite{Goodfellow2014GAN}, for generating synthetic financial data.

Variational Autoencoders (VAEs) have recently gained significant popularity in financial data generation. The first contribution in this field is the logSig-VAE, introduced by \cite{Buehler2020Add}, which utilizes a log-signature transformation and then applies the VAE in the transformed log-signature space. In parallel, \cite{desai2021timevae} proposed Time-VAE, specifically designed for predicting time series data. In the application of simulating option markets, \cite{Wiese2021MAS} combines the autoencoder structure with normalizing flows, while seamlessly integrating a no-arbitrage condition to ensure market consistency. Later, \cite{cai2023hybrid} developed a hybrid VAE to integrate the learning of local patterns and temporal
dynamics by variational inference for time series forecasting. Meanwhile, \cite{liu2022time} introduced an innovative VAE variant that bridges temporal convolutional networks and transformers through a layer-wise parallel structure, enhancing the model's ability to handle temporal sequences. Furthermore, \cite{huang2024generative} presents FTS-Diffusion, a novel VAE designed in particular to model irregular and scale-invariant patterns in time series. In addition, \cite{chung2024generative} introduce the SigMMD-VAE, which uses the signature MMD to separate distributions. More recently, \cite{schwarz2024interpretable} addresses the critical balance between model performance and interpretability by connecting ARMA-GARCH with LSTM-based VAE.

Other research mostly follows the GAN approach. This approach is first adapted to financial data generation in \cite{Takahashi2019Mft}, and later its variants are explored in \cite{efimov2020using, de2019enriching}. To better address the time dependency of financial time series, GANs with temporal structure are also explored in \cite{Wiese2020QGd, Fu2022Sft, Esteban2017RvM, Yoon2019TsG}. Conditional GANs are also studied in \cite{Koshiyama2021Gan,Coletta2021Trm}. A big family of time series GAN is based on signature transformation, e.g. in \cite{Ni2021SWG, liao2024sig, Lou2024PGg, biagini2024universal, Issa2024Nat}. In the signature-based GANs, neural SDEs introduced in \cite{Kidger2021Ond} have been proven successful. The generative expressiveness of neural SDEs is then studied in \cite{Kidger2021Ond} as infinite-dimensional GANs. Some research particularly focus on the financial quantities, see e.g. \cite{Cont2022TgL, ericson2024deep, vuletic2024fin, rizzato2023generative}. Many other approaches have also been explored. For example, \cite{Hamdouche2023GMf} suggests a novel time series generation approach using the Schrödinger bridge framework, while \cite{nagy2023generative} uses autoregressive models to generate limit order book data. In fact, a great amount of literature is dedicated to the generation of limit order book data, e.g.~\cite{hultin2023generative, cont2023limit,li2020generating,ozyar2021learning,hultin2021generative,coletta2023conditional}. 

For an extensive overview on synthetic data generation, we refer the reader to \cite{lu2023machine} (for general data), \cite{iglesias2023data} (for time series), 
\cite{eckerli2021generative} (GANs for financial time series), and \cite{assefa2020generating} (for general data in finance).

\paragraph{Notation.}
All random variables are defined on a fixed probability space $(\Omega, \mathcal{F},\sP)$, and all equalities and inequalities are intended to hold in $\sP$-almost sure sense. For $n\in\NN$, we denote by $\Pcal(\RR^n)$ the set of probability measures on $\RR^n$ and by $\Pcal_p(\RR^n)$ its subset of  measures with finite $p$-th moment, $p\in[1,\infty)$. For $m \in \R^n$ and a positive-semidefinite matrix $\Sigma \in \R^{n^2}$, we write $\calN(m,\Sigma)\in\calP(\R^n)$ for the normal distribution on $\R^n$ with mean $m$ and covariance matrix $\Sigma$, and  $\varphi_{m,\Sigma} $ for its density. For simplicity, if $m = \mathbf{0}$ and $\Sigma = \mathbf{id}_{n}$ is the identity matrix, we denote this density by $\varphi$. The \textit{entropy} of a measure $\mu\in\calP(\R^n)$ with density $p_{\mu}$ is given by
\begin{equation*}
    \sH(\mu)= -\int \log(p_{\mu}(x))p_{\mu}(x)dx.
\end{equation*}
For measures $\mu,\nu\in\calP(\R^n)$ with densities  $p_{\mu},p_{\nu}$, the \textit{Kullback–Leibler (KL) divergence} between $\mu$ and $\nu$ (or \textit{relative entropy} of $\mu$ w.r.t.~$\nu$) is given by 
\begin{equation*}
    \DKL(\mu|\nu) = 
    \int \log\Big(\frac{p_{\mu}(x)}{p_{\nu}(x)}\Big)p_{\mu}(x)dx.
\end{equation*}
For $p\in[1,\infty)$ and $\mu,\nu\in\Pcal_p(\RR^n)$, the $p$-Wasserstein distance between $\mu$ and $\nu$ is given by
\[
\Wcal_p(\mu,\nu) :=\inf_{\pi\in\Pi(\mu,\nu)}\left(\int \|x-y\|^p\pi(dx,dy)\right)^{1/p},
\]
where $\Pi(\mu,\nu)$ denotes the subset of $\Pcal(\RR^n\times\RR^n)$ of measures with first marginal $\mu$ and second marginal $\nu$. 
The elements of $\Pi(\mu,\nu)$ are called couplings of $\mu$ and $\nu$. 
For $\pi\in\Pi(\mu,\nu)$, we denote by
$\pi^x$ its kernel (disintegration) w.r.t.~$\mu$, so that $\pi(dx,dy)=\mu(dx)\pi^x(dy)$. For $n,N\in\NN$ and $\xi \in \Pcal (\RR^n)$, the push-forward measure of $\xi$ through a measurable map $T: \RR^n \rightarrow \RR^N$, denoted by $T_\# \xi$, is the probability measure on $\RR^N$ such that $T_\#\xi(A)=\xi(T^{-1}(A))$ for all Borel sets $A\in\Bcal(\RR^N)$. For $d,T\in\N$, we will look at the space $\R^{dT}$ as the collection of $d$-dimensional paths of length $T$. For $x=(x_1,\ldots,x_T)\in\R^{dT}$, we adopt the notation $x_{s:t}=(x_s,...,x_t)$, for $1\leq s\leq t\leq T$.
Moreover, we denote the up-to-time-$t$ marginal of $\mu \in \mathcal{P}(\R^{dT})$ by $\mu_{1:t}$, and the kernel of $\mu$ w.r.t.~it by $\mu_{x_{1:t}}$, so that $\mu(dx)=\mu(dx_{1:t})\mu_{x_{1:t}}(dx_{t+1:T})$.
Similarly, we denote the up-to-time-$t$ marginal of $\pi\in\Pi(\mu,\nu)$ by $\pi_{1:t}$, and the kernel of $\pi$ w.r.t.~it by $\pi_{x_{1:t}, y_{1:t}}$.

\section{Causal generator} 
\label{sect.cau_gen}
Our goal is to construct a path generator such that the generated paths are close to the observed ones, in the sense that they can be thought of as originating from the same underlying distribution. For $d,T\in\N$, we are interested in the space $\calP_1(\R^{dT})$ of distributions on $d$-dimensional paths of length $T$. We start with the observation of a set of such paths, which we consider to be an i.i.d. sample from an underlying \textit{data distribution}
$\mudata \in \calP_1(\R^{dT})$. The aim then is to build a generator that produces paths from a \textit{generated distribution} $\mugen\in \calP_1(\R^{dT})$ which we want to be as close as possible to $\mudata$. As explained in the introduction, our main motivation is data augmentation for robust decision making. The robustness results presented in Section~\ref{sect.rob} motivate us to use a modified version of the classical Wasserstein distance, called causal Wasserstein distance, which we recall in the next subsection. Afterwards, we introduce a specific structure of variational autoencoder, where we impose causality constraints on the encoder and decoder maps, and build a generator connected to it; see Figure~\ref{fig:vae}. We will show that, thanks to this causal structure, one can successfully control the causal Wasserstein distance between the data distribution and the generated one.

\subsection{Causal Distances}
\label{sect.cau_dist}
When facing stochastic optimization problems in a dynamic setting (that is, when the optimization depends on a process that evolves in time), distances from classical optimal transport (OT), such as Wasserstein distances, have proven unsuitable; see e.g.~\cite{Pflug2014Mso}.
The key observation is that, in a dynamic context, the value process depends  crucially on the conditional distributions forward in time and hence agents
take decisions accordingly. This suggests to modify  OT-distances by putting additional emphasis on conditional laws. This leads to couplings $\pi \in \Pi(\mu,\nu)$ such that the conditional law of $\pi$ is still a coupling of the conditional laws of $\mu$ and $\nu$, that is $\pi_{x_{1:t},y_{1:t}} \in \Pi\big(\mu_{x_{1:t}}, \nu_{y_{1:t}}\big)$. Such couplings are called
\textit{bi-causal} and we denote their collection by $\bccpl(\mu,\nu)$; see \cite{Lassalle2018Ctp,Pflug2012Adf,Pflug2014Mso}. We call the coupling $\pi$ \textit{causal} if $\pi_{x_{1:t},y_{1:t}} \in \Pi\big(\mu_{x_{1:t}}, \cdot\big)$, and write $\pi \in \ccpl(\mu,\nu)$, and we call it \textit{anti-causal} if $\pi_{x_{1:t},y_{1:t}} \in \cpl\big(\cdot, \nu_{y_{1:t}}\big)$, and write $\pi \in \cpl_{\mathrm{ac}}(\mu,\nu)$. Clearly, $\bccpl(\mu,\nu)=\ccpl(\mu,\nu)\cap\cpl_{\mathrm{ac}}(\mu,\nu)$. The causality constraint can be expressed in several different ways, see e.g.~\cite{BackhoffVeraguas2017Cti,Acciaio2020Cot} in the context of transport, and \cite{Bremaud1978Cof} in the filtration enlargement framework. Roughly speaking, in a causal transport, for every time $t$, only information on the $x$-coordinate up to time $t$ is used to determine the mass transported to the $y$-coordinate at time $t$.
And in a bi-causal transport this holds in both directions, i.e.~also when exchanging the role of $x$ and $y$. By means of these concepts, one can introduce constrained optimal transport problems, where the allowed couplings satisfy the causality or bicausality condition. We introduce them directly in the space of interest for us, that is $\mathcal{P}(\R^{dT})$, and for a specific cost, that is the sum over time steps of the Euclidean norm on $\R^d$.

\begin{definition}[Causal Wasserstein distance]
    The (first order) \textit{causal Wasserstein distance} $\CW_1$ on $\mathcal{P}_1(\R^{dT})$ is defined by 
    \begin{equation}\label{eq:causalwd}
        \CW_1(\mu,\nu) = \inf_{\pi \in \ccpl(\mu,\nu)}\int\sum_{t=1}^{T}\big\lVert x_{t}-y_{t}\big\rVert_{\R^{d}}\pi(dx,dy).
    \end{equation}
\end{definition}
We need to stress that calling $\CW_1$ a distance is an abuse of terminology, as this is clearly an asymmetric notion. We still adopt it as this is customary in the  literature. One way to recover symmetry, is to use bi-causal couplings instead of causal ones.
\begin{definition}[Adapted Wasserstein distance / Nested distance]
    The (first order) \textit{adapted Wasserstein distance} $\AW_1$ on $\mathcal{P}_1(\R^{dT})$ is defined by 
    \begin{equation}\label{eq:adaptedwd}
        \AW_1(\mu,\nu) = \inf_{\pi \in \bccpl(\mu,\nu)}\int\sum_{t=1}^{T}\big\lVert x_{t}-y_{t}\big\rVert_{\R^{d}}\pi(dx,dy).
    \end{equation}
\end{definition}
Bi-causal couplings and the corresponding optimal transport problem were considered by R\"{u}schendorf~\cite{Rueschendorf1985TWd} in so-called `Markov-constructions'. This concept was independently introduced by Pflug-Pichler~\cite{Pflug2012Adf} in the context of stochastic multistage optimization problems (see also \cite{Pflug2014Mso,Pflug2015Dgo,Pflug2016Feo,Glanzer2019Ism,Pichler2013Eor}), and considered by  Bion-Nadal and Talay in \cite{BionNadal2019OaW} and by Gigli in \cite{Gigli2008Otg}. 

As already mentioned above, adaptedness (or bi-causality) turns out to be the correct constraint to impose on couplings in order to modify the Wasserstein distance so to ensure robustness of a large class of stochastic optimization problems. That is to say, if two measures $\mu,\nu$ are close w.r.t.~this distance, then solving w.r.t.~$\mu$ optimization problems such as optimal stopping, optimal hedging, utility maximization etc, provides an ``almost optimizer" for $\nu$; see \cite{BackhoffVeraguas2020AWd, Pflug2014Mso}. This is not true for the Wasserstein distance, which is in fact unable to capture fundamental differences in the evolution of time series, see \cite{Pflug2014Mso}. 
We refer to Section~\ref{sect.rob} below for robustness results w.r.t.~the causal Wasserstein distance.

We conclude this subsection by showing some simple ordering between distances, that will turn out to be useful for our estimates later. For this we recall the concepts of total variation and adapted total variation, for $\mu,\nu \in \calP(\R^{dT})$:
\[
\TVD_0(\mu,\nu) = \inf_{\pi \in \cpl(\mu,\nu)}\int \mathbbm{1}_{x \neq y}\pi(dx,dy)\quad \text{and}\quad \AVD_0(\mu,\nu) = \inf_{\pi \in \bccpl(\mu,\nu)}\int \mathbbm{1}_{x \neq y}\pi(dx,dy),
\]
respectively, with $|\mu - \nu| = \mu + \nu - 2(\mu \wedge \nu)$.

\begin{lemma}
\label{thm:dom}
    Let $B_1 = \{x\in\R^{dT}\colon \Vert x \Vert \leq 1\}$. Then, for all $\mu,\nu \in \calP_1(B_1)$,
    \begin{equation}
        \CW_1(\mu,\nu) \leq \AW_1(\mu,\nu) \leq 2\AVD_0(\mu,\nu) \leq C \TVD_0(\mu,\nu) \leq C\sqrt{\frac{1}{2}\DKL(\mu|\nu)},
    \end{equation}
    where $C = 2(2^T -1)$.
    \end{lemma}
\begin{proof}
    The first inequality is obvious, and the second one is due to the fact that $\mu,\nu \in \calP_1(B_1)$. The third inequality is Lemma~3.5 in \cite{Eckstein2022Cmf} and the last inequality follows by Pinsker's inequality \cite{Takezawa2005Itn}.
\end{proof}

\subsection{TC-VAE: time-causal variational autoencoder}
\label{sect.tcvae}
\textit{Variational Autoencoders} (VAEs), introduced in \cite{Kingma2014Aev}, are \textit{deep latent-variable models} (DLVMs) employed to generate new data. Our data space is $\R^{dT}$, while as latent space we consider $\R^{d_Z T}$ for some fixed $d_Z\in \N$. A VAE consists of an encoding map (encoder) and a decoding map (decoder) where the former allows to go from the data space to the latent one, and the latter goes in the opposite direction. The encoder involves two networks $\mu_{\phi}\colon \R^{d T} \to \R^{d_Z T}$ and $\sigma_{\phi}\colon \R^{d T} \to \R^{d_Z T}$, parameterized by $\phi$. Here, $\mu_{\phi}$ encodes data $x\in\R^{dT}$ to a latent point $\mu_{\phi}(x)\in \R^{d_Z T}$, and $\sigma_{\phi}(x)$ defines the scaling of Gaussian noise that is later added to the latent point. On the other hand, the decoder map consists of a network $\De_{\theta}\colon \R^{d_Z T} \to \R^{d T}$, parameterized by $\theta$, simply mapping points from the latent space back to the data space. 
The input distribution is our data distribution $\mudata$, and we let $X\sim \mudata$. We consider Gaussian noise $\varepsilon \sim\calN(\mathbf{0},\mathbf{I}_{d_Z T})$ independent of $X$, and define 
\begin{equation}\label{VAE}
\left\{\begin{aligned}
    Z &=  \mu_{\phi}(X) + \sigma_{\phi}(X)\varepsilon,\\
    Y &= \De_\theta(Z).\\
\end{aligned}\right.
\end{equation}
We denote by $\mulatent = Z_{\#}\sP$ the distribution on the latent space $\R^{d_Z T}$ resulting from the encoding step, and by $\murecon = Y_{\#}\sP$ the reconstructed distribution on the data space $\R^{d T}$ resulting from the combination of the encoding and decoding steps.

As a DLVM, VAE generates data in the following two steps: (1) sampling $\hat{z}^{(i)} \in \R^{d_Z T}$ from a prior distribution $\muprior \in \calP_1(\R^{d_Z T})$; (2) generating a sample $\hat{x}^{(i)} \in \R^{dT}$, conditioned on $\hat{z}^{(i)}$, through the pushforward of the decoder network, i.e.~getting $\hat{x}^{(i)}$ as sample from the generated distribution 
\[
\mugen={\De_{\theta}}_{\#} \muprior \in \calP_1(\R^{dT}).
\]
Our goal is to control the causal Wasserstein distance $\CW_1(\mudata,\mugen)$, so we aim at training our VAE such that this is minimized. As we discussed already, estimating $\CW_1$  is difficult and computationally intractable due to the lack of explicit causal couplings. To overcome this difficulty, we use the fact that $\CW_1$ satisfies the triangle inequality (see \cite{Pammer2024Ano}), so that we can control $\CW_1(\mudata,\mugen)$ as follows:
\begin{equation}
\label{eq:triangle}
    \CW_1(\mudata,\mugen) \leq \CW_1(\mudata,\murecon) + \CW_1(\murecon,\mugen).
\end{equation}

Crucially, if we restrict $\mu_{\phi}$, $\sigma_{\phi}$ and $\De_\theta$ to the class of causal maps introduced below, $(X,Y)_{\#}\sP$
is a causal coupling from $\mudata$ to $\murecon$, i.e.~$(X,Y)_{\#}\sP \in \ccpl(\mudata,\murecon)$. We will show that $\CW_1(\mudata,\murecon)$ can be bounded by the transport cost associated to the coupling $(X,Y)_{\#}\sP$, 
while $\CW_1(\murecon,\mugen)$ can be  bounded by a computationally tractable quantity. 

\paragraph{Causal maps.} 
A map $\calT\colon \R^{d_1 T} \to \R^{d_2 T}$, $d_1,d_2\in\N$, is causal if and only if there exist Borel-measurable maps $\calT^t\colon \R^{d_1 t} \to \R^{d_2}$, $t=1,\ldots,T$, such that
\begin{equation*}
    \calT(x) = (\calT^{1}(x_{1:1}),\calT^{2}(x_{1:2}),\dots,\calT^{T}(x)),\qquad x \in \R^{d_1 T}.
\end{equation*}
Intuitively, the $t$-coordinate of $\calT(x)=\calT(x_{1:T})$ only depends on $x_{1:t}$, i.e.~on the values of $x$ up to time $t$. 

Let $X\sim\mu \in \calP_1(\R^{d_{1}T})$, $Y = \calT(X)$, and denote the law of $Y$ by $\nu = \calT_{\#}\mu \in \calP_1(\R^{d_{2}T})$, so that $\pi = (\mathbf{id},\calT)_{\#}\mu \in \cpl(\mu,\nu)$. Note that, for $\calT$ causal, for all $t=1,\dots,T-1$,
\begin{equation*}
\mathrm{Law}(X_{t+1}|Y_{1:t}, X_{1:t}) = \mathrm{Law}(X_{t+1}|\calT^{t}(X_{1:t}), X_{1:t}) = \mathrm{Law}(X_{t+1}|X_{1:t}),
\end{equation*}
which implies that $\pi_{x_{1:t},y_{1:t}}(dx_{t+1}) = \mu_{x_{1:t}}(dx_{t+1})$. Therefore, in this case $\pi$ is a causal coupling from $\mu$ to $\nu$, i.e.~$\pi \in \ccpl(\mu,\nu)$. Moreover, if we have a sequence of causal maps, then their composition is still a causal map by definition. This enables us to build complex maps by using simple causal maps as building blocks. 

\paragraph{Time-causal VAE (TC-VAE).} We let $\mu_{\phi}\colon \R^{d T} \to \R^{d_Z T}$, $\sigma_{\phi}\colon \R^{d T} \to \R^{d_Z T}$ and $\De_{\theta}\colon \R^{d
_Z T} \to \R^{d T}$ in VAE be \textit{causal maps} parameterized by $\phi$ and $\theta$. Let $\varepsilon \sim\calN(\mathbf{0},\mathbf{I}_{d_Z T})$  be independent of $X\sim\mudata$ and consider the autoencoder structure in \eqref{VAE};
see Figure~\ref{fig:vae} for visualization.
\begin{figure}[H]
\centering
\includegraphics[width=\textwidth]{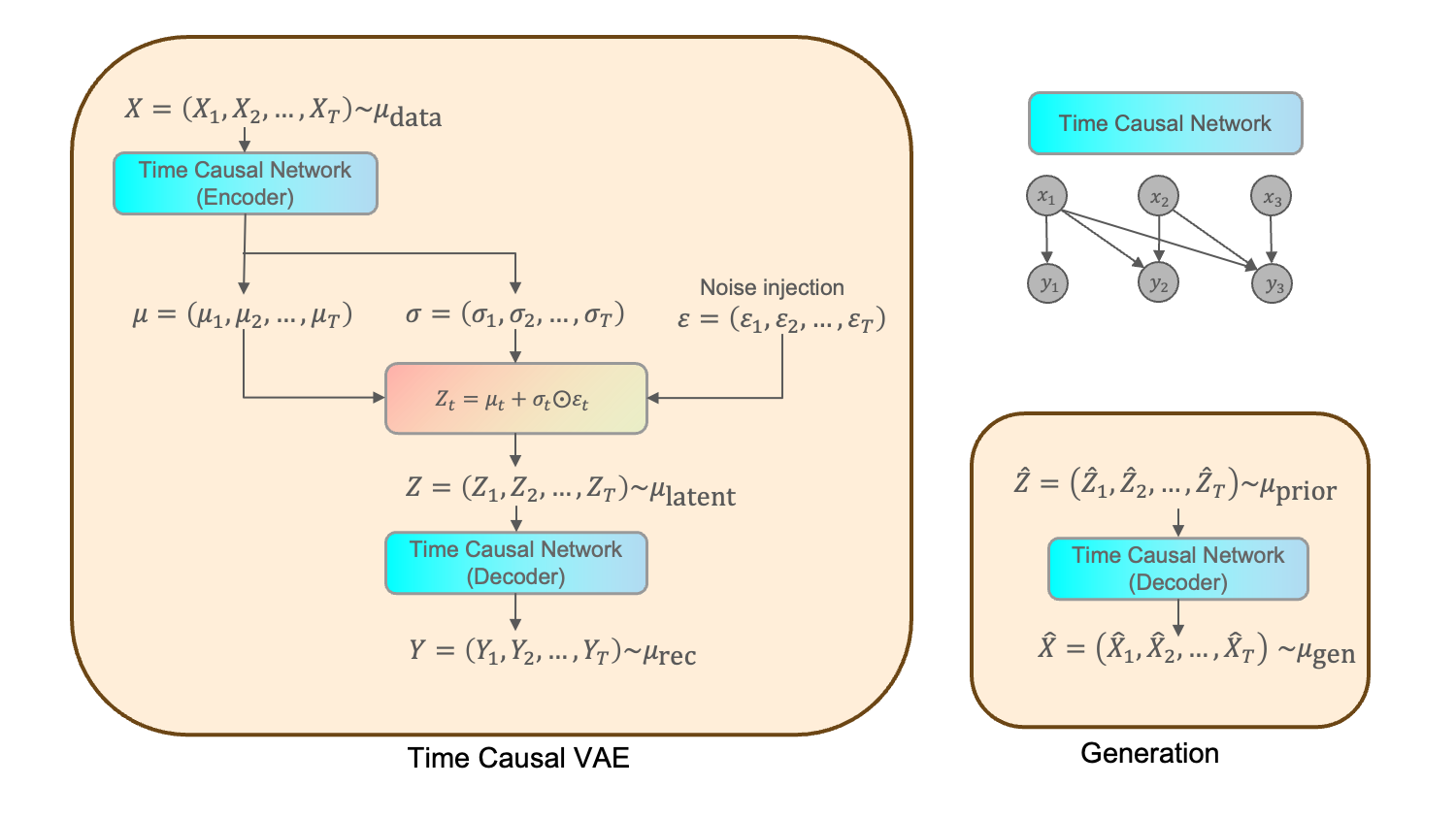}
\caption{Time-causal variational autoencoder and generation}
\label{fig:vae}
\end{figure}
Since $X$ and $\varepsilon$ are independent, we have that, for all $t=1,\dots,T-1$,
\begin{align*}
    \mathrm{Law}(X_{t+1}|Y_{1:t}, X_{1:t}) &= \mathrm{Law}(X_{t+1}|\De_{\theta}(\mu_{\phi}(X_{1:t}) + \sigma_{\phi}(X_{1:t})\varepsilon_{1:t}), X_{1:t}) = \mathrm{Law}(X_{t+1}|X_{1:t}, \varepsilon_{1:t})\\
    &= \mathrm{Law}(X_{t+1}|X_{1:t}),
\end{align*}
so that $(X,Z)_{\#}\sP$ is a causal coupling.

Now we define the reconstruction loss by 
\[
\Lrecon:=\E_{\sP}\Big[\big\Vert X - Y\big\Vert\Big]=\E_{\sP}\Big[\big\Vert X - \De_\theta(\mu_{\phi}(X) + \sigma_{\phi}(X)\varepsilon)\big\Vert\Big].
\]
Hereby, we emphasize that both $Y$ and $\Lrecon$ implicitly depend on both $\theta$ and $\phi$. Intuitively, minimizing $\Lrecon$ over the set of parameters $\phi,\theta$ corresponds to minimizing a relaxed version of $\CW_1(\mudata,\murecon)$ where causal couplings are restricted to a subset of couplings defined by our network's (causal) structure.

\begin{lemma}\label{lemma.L_rec}
We can estimate the first term in \eqref{eq:triangle} via  
 \begin{equation}
\label{eq:L_recons}
\CW_1(\mudata,\murecon) \leq   \Lrecon.
\end{equation}
\end{lemma}
\begin{proof}
Similarly as above, we can see that $\pi = (X,Y)_{\#}\sP \in \ccpl(\mudata,\murecon)$. Then, by the definition of causal Wasserstein distance,  we have the estimate
\[
\CW_1(\mudata,\murecon) \leq \E_{\sP}[\Vert X-Y\Vert] =  \E_{\sP}\Big[\big\Vert X - \De_\theta(\mu_{\phi}(X) + \sigma_{\phi}(X)\varepsilon)\big\Vert\Big].
\]
\end{proof}

Next we focus on the second term  in \eqref{eq:triangle}, that is $\CW_1(\murecon,\mugen)$. 

\begin{theorem}
\label{thm:CW_data_gen}
Assume $\mugen,\murecon \in \calP_1(B_1)$, where $B_1 = \{x\in\R^{dT}\colon \Vert x \Vert \leq 1\}$. Then
    \begin{equation}
    \label{eq:CW_recon_gen}
        \CW_1(\murecon, \mugen) \leq C\sqrt{\frac{1}{2}\DKL(\mulatent| \muprior)}.
    \end{equation}
    \end{theorem}
\begin{proof}
Lemma~\ref{thm:dom} implies
\begin{equation}
\label{eq:thm:CW_data_gen:1}
    \CW_1(\murecon, \mugen) \leq C\sqrt{\frac{1}{2}\DKL(\murecon| \mugen)}.
 \end{equation}   
Now, note that $\murecon$ and $\mugen$ originate from the same pushforward map $\De_\theta$, applied to the latent distribution $\mulatent$ and to the prior distribution $\muprior$, respectively, i.e.
\[
\murecon={\De_\theta}_\#\mulatent,\qquad \mugen={\De_\theta}_\#\muprior.
\]
Then, the data processing inequality (see e.g. \cite[Lemma 1.6]{Nutz2021Ite})  yields        \begin{equation}
    \DKL(\murecon, \mugen)\leq\DKL(\mulatent, \muprior),
    \end{equation}
which concludes the proof.
\end{proof}
Combining  \eqref{eq:triangle}, Lemma~\ref{lemma.L_rec} and Theorem~\ref{thm:CW_data_gen}, gives the following estimate for the causal Wasserstein distance between market data and generated data.
\begin{corollary}\label{coro:CW_data_gen}
 Assume $\mugen,\murecon \in \calP_1(B_1)$, where $B_1 = \{x\in\R^{dT}\colon \Vert x \Vert \leq 1\}$. Then   
    \begin{equation}
    \label{eq:CW_data_gen}
        \CW_1(\mudata, \mugen) \leq \Lrecon + C\sqrt{\frac{1}{2}\DKL(\mulatent| \muprior)}.
    \end{equation}
\end{corollary}

\begin{remark}
    The compactness assumption in Theorem~\ref{thm:CW_data_gen} and Corollary~\ref{coro:CW_data_gen} is not a constraint on our data since, in practice, we typically normalize our training data to reside within a bounded region, typically a ball.
\end{remark}

The estimate in \eqref{eq:CW_data_gen} allows us to obtain 
one sided estimates for a class of stochastic optimization problems, when employing generated rather than observed data, see Remark~\ref{rem.Vmugen} below. 

\subsection{TC-VAE Loss}
\label{sect.tcvae_loss}
Motivated by Corollary~\ref{coro:CW_data_gen}, we aim at training our TC-VAE by minimizing the two terms on the RHS of \eqref{eq:CW_data_gen}: $\Lrecon$ and $\DKL(\mulatent| \muprior)$. The first term $\Lrecon$, defined in \eqref{eq:L_recons}, is the well-known reconstruction loss and can be efficiently estimated by taking batch-wise sample expectations. However, the second term $\DKL(\mulatent| \muprior)$ involves an intractable term $\mulatent$, also called the \textit{aggregated posterior distribution}. Recall that $\mulatent = Z_{\#}\sP = (\mu_{\phi}(X) + \sigma_{\phi}(X)\varepsilon)_{\#}\sP$ where $X\sim\mudata$ and $\varepsilon \sim\calN(\mathbf{0},\mathbf{I}_{d_Z T})$ are independent. We denote the density of $\mulatent$ and $\muprior$ by $p_{\mathrm{latent}}$ and $p_{\mathrm{prior}}$, respectively. Then we have $p_{\mathrm{latent}}(\cdot) = \int q_{\phi}(\cdot|x)\mudata(dx)$, where $q_{\phi}(\cdot|x)$ is the density of $\calN(\mu_{\phi}(x),\Sigma_\phi(x))$ and $\Sigma_\phi(x)$ is the diagonal matrix with diagonal vector $\sigma_{\phi}(x)$ for $x\in \R^{dT}$. 
Note that the integral form of $\platent$ makes it computationally intractable, 
and so is $\DKL(\mulatent| \muprior)$. Luckily,  $\DKL(\mulatent| \muprior)$ can be bounded by $\int \DKL(q_{\phi}(\cdot|x)|\pprior) \mudata(dx)$, following the same arguments used in \cite{Kingma2014Aev} to bound  log-likelihood by evidence lower bound:
\begin{equation}
\label{eq:L_regu}
\begin{split}
    \DKL(\mulatent| \muprior) &= - \int \log(\pprior(z))\platent(z)dz - \sH(\platent) \\
    &= -\int \log(\pprior(z))\int q_{\phi}(z|x)\mudata(dx)dz - \sH(\platent) \\
    &= \int \DKL(q_{\phi}(\cdot|x)|\pprior) \mudata(dx)+ \int \sH(q_{\phi}(\cdot|x))\mudata(dx) - \sH\bigg(\int q_{\phi}(\cdot|x)\mudata(dx)\bigg)\\
    &\leq \int \DKL(q_{\phi}(\cdot|x)|\pprior) \mudata(dx) =  \E_{\sP}[\DKL(q_{\phi}(\cdot|X)|\pprior)] \eqqcolon \Llatent,
\end{split}
\end{equation}
where last inequality follows by Jensen's inequality, and $\Llatent$ stands for ``latent loss". Here again, for  simplicity, we omit $\phi$ in the notation of $\Llatent$. 
Now, instead of minimizing $\DKL(\mulatent| \muprior)$, we minimize $\Llatent$, which is tractable. More precisely,
we follow the loss design in $\beta$-VAE in \cite{Higgins2016bvL} by setting $\calL_{\theta,\phi}:=\Lrecon + \beta \Llatent$ and solving
\begin{equation*}
   \min_{\theta,\phi} \calL_{\theta,\phi},
\end{equation*}
where $\beta$ is an hyper-parameter. We now focus on computing $\calL_{\theta,\phi}$. Notice that $q_{\phi}(\cdot|x)$ is the density of $\mathrm{Law}(Z|X=x)$, so we can rewrite
\begin{equation*}
   \Llatent = \E_{\sP}\Big[\E_{Z\sim q_{\phi(\cdot|X)}}\big[\log q_{\phi}(Z|X) - \log \pprior(Z)\big]\Big] = \E_{\sP}\Big[\log q_{\phi}(Z|X) - \log \pprior(Z)\Big].
\end{equation*}
This yields
\begin{equation*}
    \calL_{\theta,\phi} = \E_{(X,\epsilon)\sim\mudata\otimes \calN(\mathbf{0},\mathbf{I}_{d_Z T}), Z = \mu_{\phi}(X) + \sigma_{\phi}(X)\varepsilon}\Big[\big\Vert X - \De_\theta(Z)\big\Vert + \beta \,\Big(\log q_{\phi}(Z|X) - \log \pprior(Z)\Big)\Big].
\end{equation*}
\paragraph{Sample-based loss.} 
In practice, we have no access to $\mudata$, but only to finitely many samples from observation. We let $x^{(i)}\in\R^{dT}, i=1,\ldots,n$, be i.i.d.~samples from $\mudata \in \calP_1(\R^{dT})$, and $\muemp = \frac{1}{n}\sum_{i=1}^{n}\delta_{x^{(i)}}$ be the corresponding empirical distribution. So we take expectation under $\muemp$ instead of $\mudata$ and minimize a sample-based version of the loss:
\begin{equation*}
    \calL^n_{\theta, \phi} =  \frac{1}{n}\sum_{i=1}^{n}\big\Vert x^{(i)} - \De_\theta(z^{(i)})\big\Vert+ \beta\,\frac{1}{n}\sum_{i=1}^{n} \Big(\log q_{\phi}(z^{(i)}|x^{(i)}) - \log \pprior(z^{(i)})\Big),
\end{equation*}
where $\{(x^{(i)},\epsilon^{(i)})\}_{i=1,\ldots,n}$  are i.i.d. samples from $\mudata \otimes \calN(\mathbf{0},\mathbf{I}_{d_Z T})$, and $z^{(i)} = \mu_{\phi}(x^{(i)}) + \sigma_{\phi}(x^{(i)})\epsilon^{(i)}$ for $i=1,\dots,n$.

\subsection{Flow-based prior distribution} 
\label{sect.tcvae_prior}
In the standard VAE, the prior distribution $\muprior$ is a fix standard Gaussian prior, hence non-learnable, and the regularization term pushes the aggregated posterior to match it. In practice, the matching is not satisfied because  $\Lrecon$ forces the encoder to be irregular and, in the end, it is almost impossible to precisely match a fixed-shaped prior. As a result, one obtains `holes', namely regions in the latent space where the aggregated posterior assigns low probability while the prior assigns relatively high probability. This is an issue in generation because sampling from the prior, from the hole, may result in a sample that is of extremely low quality \cite{Rezende2018Tv}. It is even more problematic in our application, because we are not only interested in generating samples on the manifold where empirical data lies, but also generating a distribution close to the empirical measure under some strong probabilistic metric. Therefore, in the latent space, we not  only require that the aggregated posterior has no holes, but actually need that the aggregated posterior is close to the prior. 

Sampling directly from the posterior distribution at first glance seems to solve this issue, but it potentially leads to  overfitting of the empirical data. Therefore, we consider a learnable prior to match the posterior distribution. There are many options such as mixture of Gaussians \cite{dilokthanakul2016deep}, VampPrior \cite{tomczak2018vae}, generative topographic mapping \cite{bishop1998gtm}, flow-based prior \cite{chen2016variational}, etc. In our case, we use the flow-based prior because of its high flexibility with time series data. Let $Z_0$ be a random variable on $\R^{d_{Z} T}$ 
with density $p_0(z_0)$. Let $f_1\colon \R^{d_Z T} \to \R^{d_Z T}$ be invertible, set $Z_1 = f_1(Z_0)$ and denote its density by $p_1(z_1)$. Then, for all $z_1 = f_1(z_0)$, $z_0 \in \R^{d_Z T}$,
\begin{equation*}
    p_1(z_1) = p_0(z_0)\Big\vert \mathrm{det} \frac{dz_0}{dz_1} \Big\vert = p_0(z_0)\Big\vert \mathrm{det}\nabla f_1^{-1}(z_1)\Big\vert,
\end{equation*}
where $\nabla f_1^{-1}(z_1)$ is the Jacobian of $f_1^{-1}$ at $z_1$. Now, let $f_1,\dots,f_N$ be a sequence of invertible functions, and set $f = f_1\circ\dots\circ f_N$. Let $Z_{j} = f_{j}(Z_{j-1})$ and denote its density by $p_{j}(z_j)$, for $j=1,\dots,N$. Then, for all $j=1,\dots,N$, $z_j = f_j(z_{j-1})$, $z_0 \in \R^{d_Z T}$, we have
\begin{equation*}
    p_j(z_j) = p_{j-1}(z_{j-1}) \Big\vert \mathrm{det}\nabla f_j^{-1}(z_{j})\Big\vert .
\end{equation*}
Given such a chain of probability density functions, we have 
\begin{equation*}
    \log(p_N(z_N)) = \log\bigg( p_0(z_{0}) \prod_{j=1}^{N}\Big\vert \mathrm{det} \nabla f_j^{-1}(z_{j})\Big\vert \bigg) = \log(p_0(z_{0})) + \sum_{j=1}^{N} \log\bigg(\Big\vert \mathrm{det} \nabla f_j^{-1}(z_{j})\Big\vert \bigg).
\end{equation*}
Next, we parameterize $f^{\lambda}_1,\dots,f^{\lambda}_N$ by a parameter $\lambda$ and let $f^{\lambda} = f_1^{\lambda}\circ\dots\circ f_N^{\lambda}$. Then $Z_N = f^{\lambda}(Z_0)$ has a learnable density $p_\lambda(z)$. 

\paragraph{Sample-based loss with flow-based prior distribution.} 
Setting $\pprior = p_\lambda$ in our TC-VAE, we end up with a learnable prior, and $\log \pprior(z)$ in the regularization term $\Llatent$ becomes 
\begin{equation*}
    \log \pprior(z) = \log(p_0(z_{0})) + \sum_{j=1}^{N} \log\bigg(\Big\vert \mathrm{det} \nabla{f_j^{\lambda}}^{-1}(z_{j})\Big\vert \bigg),
\end{equation*}
where $z = z_N$, $z_j = f_j^{\lambda}(z_{j-1})$ for all $j=1,\dots,N$, $z_0 \in \R^{d_Z T}$. 
We end up minimizing the following loss:
\begin{equation*}
    \calL^n_{\theta, \phi, \lambda} =  \frac{1}{n}\sum_{i=1}^{n}\big\Vert x^{(i)} - \De_\theta(z^{(i)})\big\Vert + \beta\,\frac{1}{n}\sum_{i=1}^{n}\log q_{\phi}(z^{(i)}|x^{(i)}) - \beta\,\frac{1}{n}\sum_{i=1}^{n}\bigg(\log(p_0(z_{0}^{(i)})) + \sum_{j=1}^{N} \log\Big(\Big\vert \mathrm{det} \nabla{f_j^{\lambda}}^{-1}(z_{j}^{(i)})\Big\vert \Big)\bigg),
\end{equation*}
where $z_N^{(i)} = z^{(i)}$, $z_j^{(i)} = f_j^{\lambda}(z_{j-1}^{(i)})$ for all $i=1,\dots,N$. In practice, we choose the sequence of invertible transformations from the class of real-valued non-volume-preserving  (real NVP) transformations; see \cite{Dinh2016Deu}.

\begin{remark}
In the flow-based prior, we generalize the neural SDE \cite{Kidger2021Nsa} with noise process driven by a richer class of distribution beyond the Gaussian noise.
\end{remark}

\section{Causal Robustness}
\label{sect.rob}
In this section, we show why causal (resp.~adapted) Wasserstein distance is a suitable way to measure closeness when considering a broad class of stochastic optimization problems. This is done by establishing one-sided (resp.~two-sided) robustness of controlled problems under this distance. Let $\calQ \colon \R^{dT}\times \R^{dT} \to \R$ be a measurable function. For $\mu \in \calP_1(\R^{dT})$, we consider the following  stochastic optimization problem:
\begin{equation}\label{eq.emp}
    \calV(\mu) := \inf_{H \in \calH} V(H,\mu), \quad V(H,\mu) =  \int \calQ(x,H(x))\mu(dx) = \E_{X\sim \mu}[\calQ(X,H(X))],
\end{equation}
where $\Hcal$ is a closed convex subset of
 the set $\Kcal$ of adapted strategies (controls):
\[
\Kcal:=\{H = \left(H_t\right)_{t=1}^{T} \colon H_t(x) = H_t(x_{1:t}), H_t\colon \R^{dT} \to \R^{d}\; \text{measurable}\}.
\]
Here, with an abuse of notation, we write $H_t(x_{1:t})$, meaning that the function $H_t$ defined on $\R^{dT}$ only depends on the first $t$ coordinates of the argument, i.e. $H$ is adapted.
\begin{theorem}
\label{thm:causal_risk}
Let $L\geq 0$ and $\calQ$ be s.t. $(x,h)\mapsto\calQ(x,h)$ is uniformly $L$-Lipschitz in $x$ and convex in $h$. Then, for all $\mu, \nu \in \calP_1(\R^{dT})$,
    \begin{equation*}
    \begin{split}
        \calV(\mu) - \calV(\nu) &\leq L\,\CW_1(\mu,\nu),\\[0.1cm]
        \vert \calV(\mu) - \calV(\nu) \vert &\leq L\,\AW_1(\mu,\nu).
    \end{split}
    \end{equation*}
\end{theorem}

\begin{proof} 
Let $\pi \in \ccpl(\mu,\nu)$ be an optimal coupling for $\CW_1(\mu,\nu)$.
 Notice that, for all $G \in \calH$, we have
    \begin{equation}
    \label{eq:thm:causal_risk:1}
    \begin{split}
        - V(G,\nu) &= - \int \calQ(y,G(y))\nu(dy) = -\int  \calQ(y,G(y))\pi(dx,dy)     \\
        &=\int \left[\calQ(x,G(y)) -\calQ(y,G(y))\right]\pi(dx,dy)-\int \calQ(x,G(y))\pi(dx,dy) .
    \end{split}
    \end{equation}
    We first estimate the first term on the RHS of \eqref{eq:thm:causal_risk:1}. By the Lipschitz property of $\calQ$, we have 
    \begin{equation}\label{eq:thm:causal_risk:2}
        \begin{split}
            \int \left[\calQ(x,G(y)) -\calQ(y,G(y))\right]\pi(dx,dy) &\leq L \int \Vert x - y\Vert \pi(dx,dy) = L \cdot \CW_1(\mu,\nu).
        \end{split}
    \end{equation}
    Next, we estimate the second term on the RHS of \eqref{eq:thm:causal_risk:1}. By convexity of $\calQ$ and Jensen's inequality, we have
    \begin{equation*}
    \begin{split}
         -\int \calQ(x,G(y))d\pi &= -\int\int  \calQ(x,G(y))\pi_{x}(dy)\mu(dx)\\
         &\leq -\int  \calQ\!\left(x,\int G(y)\pi_{x}(dy)\right)\mu(dx) = -\int    \calQ\!\left(x,\tilde{G}(x)\right)\mu(dx),
    \end{split}
    \end{equation*}
   where $\tilde{G} = (\tilde{G}_t)_{t=1}^{T}\colon \R^{dT} \to \R^{dT}$ is defined as $\tilde{G}(x) \defeq\int G(y)\pi_{x}(dy)$. Then, since $\pi \in \ccpl(\mu,\nu)$, for all $x = x_{1:T} \in \R^{dT}$ and $t=1,\dots,T$, we have
    \begin{equation*}
        \tilde{G}_t(x) = \int G_t(y)\pi_{x}(dy) = \int G_t(y_{1:t})\pi_{x}(dy_{1:t}) = \int G_t(y_{1:t})\pi_{x_{1:t}}(dy_{1:t}) = \tilde{G}_t(x_{1:t}),
    \end{equation*}
    where the second equality follows by adaptedness of $G$, and the third one by causality of $\pi$. This implies that $\tilde{G}$ is also an adapted strategy, and thus $\tilde{G} \in \calH$. 
    Therefore, for the second term on the RHS of \eqref{eq:thm:causal_risk:1}, we have
    \begin{equation}\label{eq:thm:causal_risk:3}
        -\int \calQ(x,G(y))\pi(dx,dy) \leq  - V(\tilde{G},\mu) \leq  - \calV(\mu).
    \end{equation}
    Combining \eqref{eq:thm:causal_risk:1}, \eqref{eq:thm:causal_risk:2} and \eqref{eq:thm:causal_risk:3}, for all $H,G\in\calH$ we have that
    \begin{equation*}
         - V(G,\nu) \leq L \cdot \CW_1(\mu,\nu) - \calV(\mu),
    \end{equation*}
    so that
    \[
    \calV(\mu)- V(G,\nu) \leq L \cdot \CW_1(\mu,\nu).
    \]
    Then, by the arbitrarity of $G\in\calH$, we conclude that
    \begin{equation*}
        \calV(\mu) - \calV(\nu) \leq L \cdot \CW_1(\mu,\nu).
    \end{equation*}
    By symmetry, the claimed inequality for $\AW_1$ holds as well. This completes the proof.
\end{proof}

\begin{remark}\label{rem.Vmugen} Thanks to the estimate obtained in
Corollary~\ref{coro:CW_data_gen}, Theorem~\ref{thm:causal_risk} provides one-sided estimates for stochastic optimization problems as in \eqref{eq.emp}, when employing generated rather than observed data. 
Specifically, for $\mugen,\murecon \in \calP_1(B_1)$, $B_1 = \{x\in\R^{dT}\colon \Vert x \Vert \leq 1\}$, and $\calQ$ as in Theorem~\ref{thm:causal_risk}, we have
\begin{equation*}
\begin{split}
\calV(\mudata) & \leq \calV(\mugen) + L\bigg(\Lrecon + C\sqrt{\frac{1}{2}\DKL(\mulatent| \muprior)}\bigg).
\end{split}
\end{equation*}
This means that, if we solve the minimization problem on the samples we generate, the optimal value, together with the reconstruction loss and the latent loss, gives a conservative upper bound of the optimal value $\calV(\mudata)$ under the data distribution. 
\end{remark}

In order to provide an example of application of the above theorem, let us recall the definitions of Value at Risk and Average Value at Risk.
\begin{definition}[VaR and AVaR]
Let $\alpha \in (0,1)$ and $U$ be a real random variable on $(\Omega, \calF, \sP)$. The Value-at-Risk (VaR) of $U$ at confidence level $\alpha$ is defined as the negative $\alpha$-quantile of $U$ under $\sP$:
\begin{equation*}
    \var_{\alpha}(U) = -\inf\{x\in\R \colon \sP(U\leq x) \geq \alpha\}.
\end{equation*}  
The Average Value at Risk (or Expected Shortfall) of $U$ at confidence level $\alpha$ is defined as the average of the VaR below $\alpha$: 
\begin{equation*}
\avar_{\alpha}(U) = \es_{\alpha}(U) = \frac{1}{\alpha}\int_0^\alpha \var_u(U)du= \min_{z \in \mathbb{R}}\left\{\frac{1}{\alpha}E[(z-U)_{+}] -z\right\}.
\end{equation*}
\end{definition}

\begin{corollary}
\label{cor:causal_risk_avar}
Let $L\geq 0$ and $\calQ$ be s.t. $(x,h)\mapsto\calQ(x,h)$ is uniformly $L$-Lipschitz in $x$ and concave in $h$. For $\alpha \in (0,1)$ and $\mu \in \calP_1(\R^{dT})$, define 
    \begin{equation*}
        \calR_{\alpha}(\mu) = \inf_{H\in\calK}\es_{\alpha}(\calQ(X,H(X))),\quad \text{where } X=(X_1,\ldots,X_T) \sim \mu.
    \end{equation*}
    Then, for all $\mu, \nu \in \calP_1(\R^{dT})$,
    \begin{equation}\label{eq.es_est}
    \begin{split}
        \calR_{\alpha}(\mu) - \calR_{\alpha}(\nu) &\leq \frac{L}{\alpha} \CW_1(\mu,\nu),\\
        \vert \calR_{\alpha}(\mu) - \calR_{\alpha}(\nu) \vert &\leq \frac{L}{\alpha} \AW_1(\mu,\nu).
    \end{split}
    \end{equation}
\end{corollary}

\begin{proof}
By the dual representation of the expected shortfall, we can rewrite $\calR_{\alpha}(\mu)$ as a stochastic optimization problem on an enlarged space: 
\begin{equation*}
    \calR_{\alpha}(\mu) = \inf_{H\in\calK} \inf_{z\in\R}\E\big[\frac{1}{\alpha}(z - \calQ(X,H(X)))_{+} - z\big] = \inf_{\tilde{H}\in\tilde{\calH}} \E[\tilde{\calQ}_{\alpha}(\tilde{X},\tilde{H}(\tilde{X}))],
\end{equation*}
with
\begin{align*}
        \tilde{\calQ}_{\alpha}(\tilde{x},\tilde{h}) &= \frac{1}{\alpha}(\tilde{h}_1^{(1)} - \calQ(\tilde{x}^{(2:d+1)},\tilde{h}^{(2:d+1)}))_{+} - \tilde{h}_1^{(1)},\\
        \tilde{X} &= ([1, X_t]^\top)_{t=1}^{T} \in \R^{(d+1)T},\\
        \tilde{\calH} &= \{\tilde{H} \colon \tilde{H}_t(\tilde{x}) = [z, H_t(\tilde{x}^{(2:d+1)})]^\top, {H\in\calK}, z\in\R\},
\end{align*}
where, for $v=[v_1,\ldots,v_m]^T$, we use the notation $v^{(i)}=v_i$ and $v^{(i:j)}=(v_i,\ldots,v_j)$.
On the one hand, $\tilde{\calQ}_{\alpha}$ is uniformly $\frac{L}{\alpha}$-Lipschitz in the first argument, because $u \mapsto \frac{1}{\alpha}u_{+}$ is $\frac{1}{\alpha}$-Lipschitz and $\calQ$ is uniformly $L$-Lipschitz in the first argument. On the other hand, $\tilde{Q}_{\alpha}$ is convex in the second argument, because $u \mapsto \frac{1}{\alpha}u_{+}$ is convex and $-Q$ is convex in the second argument. Therefore, we can apply Theorem~\ref{thm:causal_risk} and complete the proof. 
\end{proof}
A concrete and practical example of a function $Q$ satisfying the assumptions in Theorem~\ref{thm:causal_risk} is the profit and loss function for bounded strategies.
\begin{example}[P\&L]
Consider $\calQ(x, h) = \sum_{t=1}^{T-1} h_t(x_{t+1} - x_{t}).$ Then $\calQ$ is linear  in $h$. Also notice that
\begin{equation*}
    \Big\Vert \sum_{t=1}^{T-1} h_t(x_{t+1} - x_{t}) - \sum_{t=1}^{T-1} h_t(x'_{t+1} - x'_{t})\Big\Vert \leq \sum_{t=1}^{T-1}\Vert h \Vert_{\infty}(\Vert x'_{t+1} - x_{t+1}\Vert + \Vert x'_{t} - x_{t}\Vert) \leq 2\Vert h \Vert_{\infty} \Vert x - x' \Vert.
\end{equation*}
For $B \geq 0$, then, over all $h \in \R^{dT}$ s.t. $\Vert h \Vert_{\infty} \leq B$, $\calQ$ is uniformly $2B$-Lipschitz in $x$. Therefore, for all $\alpha\in (0,1)$, if we consider the risk minimization problem with $B$-bounded strategies:
\begin{equation*}
    \calR_{\alpha}^{B}(\mu) = \inf_{H\in\calH_B}\es_{\alpha}(\calQ(X,H(X))) = \inf_{H\in\calH_B}\es_{\alpha}\big(\sum_{t=1}^{T-1}H_t(X_{1:t})(X_{t+1} - X_t)\big),
\end{equation*}
where $\calH_{B} = \{{H\in \calK} \colon \Vert H \Vert_{\infty} \leq B\}$, then, by Corollary~\ref{cor:causal_risk_avar}, the following holds true: 
\begin{equation}\label{eq.es_est2}
\begin{split}
    \calR_{\alpha}^{B}(\mu) - \calR_{\alpha}^{B}(\nu) &\leq \frac{2B}{\alpha} \CW_1(\mu,\nu),\\
    \vert \calR_{\alpha}^{B}(\mu) - \calR_{\alpha}^{B}(\nu) \vert &\leq \frac{2B}{\alpha} \AW_1(\mu,\nu).
\end{split}
\end{equation}
\end{example}

To conclude, while the Wasserstein distance is not strong enough to guarantee closeness in the performance of stochastic optimization problems (see \cite{Pflug2014Mso}), we show that the causal Wasserstein distance can control the closeness in a Lipschitz fashion. The above results can be regarded as asymmetric versions of robustness results known for the adapted Wasserstein distance; see \cite{BackhoffVeraguas2020AWd}.

\section{Experiments}
\label{sec.exp}
In this section, we test the generative capabilities  of TC-VAE in three major aspects. First, we compare the generated data and market data under weak metrics \cite{Rachev2013Tmo}. These include financial statistics \cite{Cont2001Epo}, Wasserstein distance \cite{Arjovsky2017WGA}, and different maximum mean discrepancies \cite{Gretton2012Akt}. Second, we compare the generated data and market data under adapted metrics \cite{BackhoffVeraguas2020AWd}. These include adapted Wasserstein distance \cite{BackhoffVeraguas2020AWd} and the optimal value of multistage optimization problems, like portfolio optimization \cite{Forsyth2022MPM}, utility maximization \cite{Merton1975Oca}, and optimal stopping \cite{Becker2019Dos}. Finally, we evaluate the diversity in generated data compared to market data. This tests how much we enlarge the market dataset with new samples. Throughout this section, we also refer to market data as real paths and to generated data as fake paths.

\subsection{Synthetic data}

\subsubsection{Black-Scholes model}
\label{sect.bs}
The Black-Scholes model \cite{Black1973Tpo} is the most renowned model in mathematical finance. Various stochastic optimization problems have analytic optimal solutions under this model \cite{Li2000Odp, Merton1975Oca}. This provides benchmarks for the comparison between market and generated data. Let $(S^{\mathrm{BS}}_t)_{t\geq 0}$ be defined by $S^{\mathrm{BS}}_0 = 1$ and $dS^{\mathrm{BS}}_t = S^{\mathrm{BS}}_t(\mu dt + \sigma dW_t)$ for all $t\geq 0$, with drift $\mu = 0.1$, volatility $\sigma = 0.2$, and where $(W_t)_{t\geq 0}$ is the Wiener process. Since we work in discrete time, we let $\Delta t = 1/12$, $T=5$, $N_T = T/\Delta t$ to model monthly prices over an investment horizon of $5$ years. We choose $N=1000$ to be the number of samples in market data, reflecting the scarcity of data in practice \cite{Buehler2020Add}. 

\subsubsection*{Evaluation under weak metrics}
First, we visually compare real paths and fake paths to provide a proof-of-concept for TC-VAE. Figure~\ref{fig:bs_real_fake} illustrates that fake paths are visually indistinguishable from real paths. 
\begin{figure}[H]
\centering
\includegraphics[width=1\textwidth]{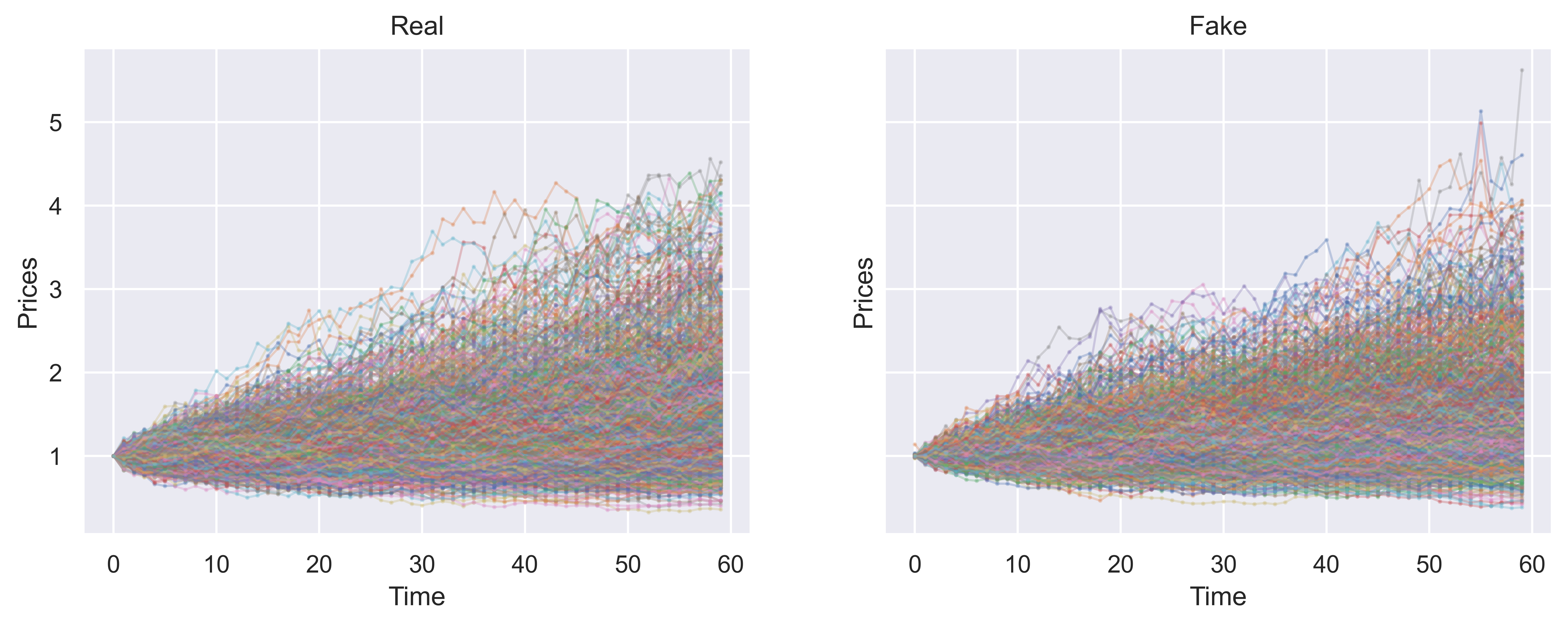}
\caption{Illustration of real paths from a discretized Black-Scholes model (left) compared to fake paths generated from the TC-VAE model (right).}
\label{fig:bs_real_fake}
\end{figure}
Then we compare the marginal histograms between real and fake paths. Figure~\ref{fig:bs_real_fake_marginal} shows that real and fake paths are close in one-dimensional marginal distributions.
\begin{figure}[H]
\centering
\includegraphics[width=\textwidth]{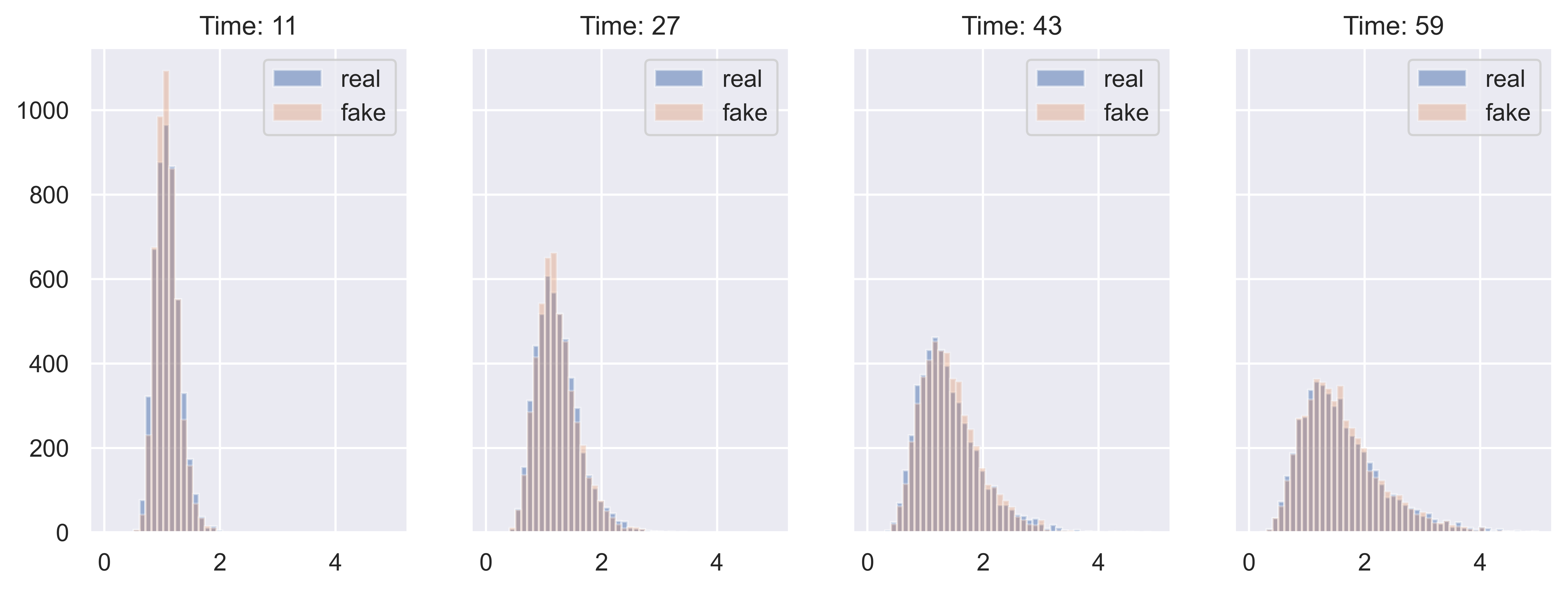}
\caption{Visualization of marginal distributions at different time slices for real paths from a discretized Black-Scholes model (blue) compared to fake paths generated from the TC-VAE model (orange).}
\label{fig:bs_real_fake_marginal}
\end{figure}
Furthermore, we compare drift and volatility between real and fake paths. For sample paths $(S^{(n)}_{j\Delta t})_{\substack{j=0,\dots,N_T\\\!\!\!\!n = 1,\dots, N}}$, we compute the log-paths and the volatility: 
\[
\left(\log S^{(n)}_{j\Delta t}\right)_{\substack{j=0,\dots,N_T\\\!\!\!\!n = 1,\dots, N}}, \quad \left(\sqrt{\frac{1}{N_T \Delta t}\sum_{j=1}^{N_T}\left(\log S^{(n)}_{j\Delta t} - \log S^{(n)}_{(j-1)\Delta t}\right)^2}\,\,\right)_{n=1,\dots,N}.
\]
We compare the log-path distribution and the volatility distribution for both real and fake paths; see Figure~\ref{fig:bs_drift_volatility}. Overall real and fake paths are close in drift and volatility.
\begin{figure}[H]
    \centering
    \includegraphics[width=\linewidth]{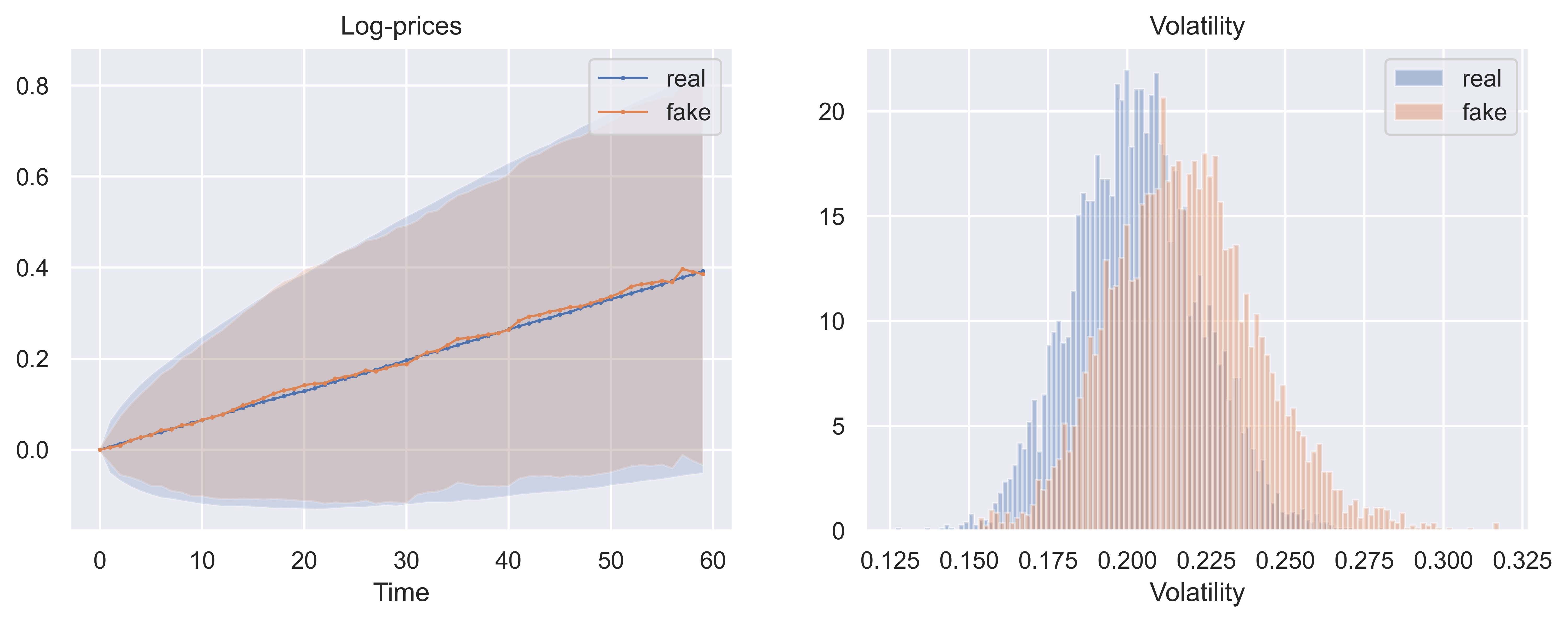}
    \caption{The left-hand side visualizes the log-paths for real prices (blue) compared to fake prices (orange). The solid lines represent the mean and the widths of shadow areas represent the standard deviation. The right-hand side visualizes the histogram of the volatility of real paths (blue) compared to fake paths (orange).}
    \label{fig:bs_drift_volatility}
\end{figure}

Next, we compute the sliced Wasserstein distance, which is a principled method to simultaneously compare all one-dimensional  projected distributions between two measures (see \cite{Kolouri2019Gsw}). We compute the sliced Wasserstein distance between real and fake paths with $10$ realization. To benchmark, we compute the sliced Wasserstein distance between real and control paths, where control paths are discretized Black-Scholes paths different from real paths only in volatility. Moreover, we conduct the same evaluation under Gaussian kernels MMD \cite{Gretton2012Akt} and signature MMD \cite{Chevyrev2018Smt}, which are widely used to evaluate the discrepancy between probability measures \cite{Li2017MgT,Ni2021SWG}. In Figure~\ref{fig:mmd_swd}, we show that real and fake paths are relatively close in sliced Wasserstein distance, Gaussian MMD, and signature MMD. As a comparison, we train the unconditional Sig-VAE introduced in \cite{chung2024generative} using the same Black-Scholes distribution. We utilize the code from the authors’ repository available at \url{https://github.com/luchungi/Generative-Model-Signature-MMD}.

\begin{figure}[H]
    \centering
    \includegraphics[width=\linewidth]{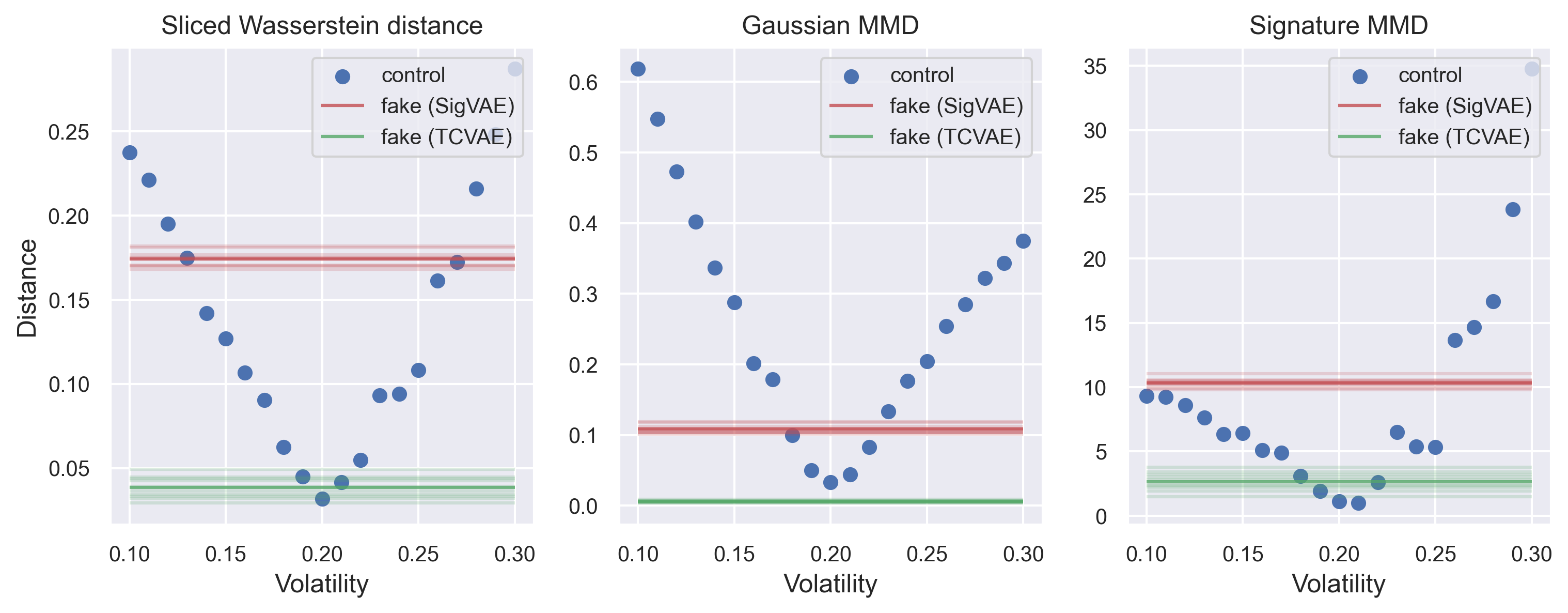}
    
    \caption{From left to right, we visualize the sliced Wasserstein distance, Gaussian MMD, and signature MMD. The green (resp.~red) lines illustrate distances between real paths of the Black-Scholes model and fake paths generated from TC-VAE (resp.~Sig-VAE); each line froms a different random seed. The blue dots show the distances between real paths and control paths under different volatility levels.}
    \label{fig:mmd_swd}
\end{figure}

Next, we compare real and fake paths using adapted metrics. In particular, we study the mean-variance portfolio optimization problem, the log-utility maximization problem, the optimal stopping problem and the adapted Wasserstein distance. 

\subsubsection*{Mean-variance portfolio optimization problem}
The mean-variance portfolio optimization problem is one of the most classical and frequently used portfolio selection rules \cite{Rubinstein2002Mps}. The investor aims at maximizing the expected return from terminal wealth while minimizing the variance of the portfolio:
\begin{equation*}
    \calV(\mu) = \sup_{\theta} \E_{\mu}[V^{\theta}_T] -\kappa \mathrm{Var}_{\mu}[V^{\theta}_T].
\end{equation*}
The supremum is taken over the set of self-financing trading strategies (see \cite[Definition 5.4]{Foellmer2011Sfa}) and $\kappa$ is the hyper-parameter. $V^\theta_T$ is the corresponding terminal value of the value process (see \cite[Definition 5.6]{Foellmer2011Sfa}), where the market consists of one risky asset (either real, fake or control paths in our case) and a risk-free asset $S^{\mathrm{bond}}_t = e^{rt}$, $t\geq 0$, with risk-free rate $r=0.01$. In Section~\ref{sect.rob}, we showed that this problem is a multistage optimal control problem, for which the optimal values are Lipschitz w.r.t.~the adapted Wasserstein distance. 

As a first comparison, we consider constant proportional trading strategies. For each constant strategy, we calculate mean and variance of the corresponding terminal wealth. By varying proportions, we obtain the efficient frontier of constant proportional trading strategies. We compare such efficient frontiers for real, fake and control paths. Control paths, as before, are discretized Black-Scholes paths different from real paths only in volatility. The efficient frontier by fake paths matches extremely well to the one by real paths, see the left-hand side of Figure~\ref{fig:mean_variance_portfolio}. 

Next we compare the efficient frontiers under optimal strategies. For Black-Scholes paths (real and control paths), we know the analytical optimal strategies \cite{Li2000Odp, Forsyth2022MPM}, so that we can directly calculate the corresponding efficient frontiers. For fake paths, we restrict the optimization to optimal strategies of the Black-Scholes model. So we compute efficient frontiers with the optimal strategies of Black-Scholes under different volatility $\sigma$, denoted by $\theta_\sigma$. Then we take the supremum of all efficient frontiers as the efficient frontier of fake paths. The efficient frontier of fake paths matches again extremely well to the one by real paths, see the right-hand side of Figure~\ref{fig:mean_variance_portfolio}. This implies that the performance of mean-variance portfolio optimization problem are close between real and fake paths.

\begin{figure}[H]
\centering
\includegraphics[width=\textwidth]{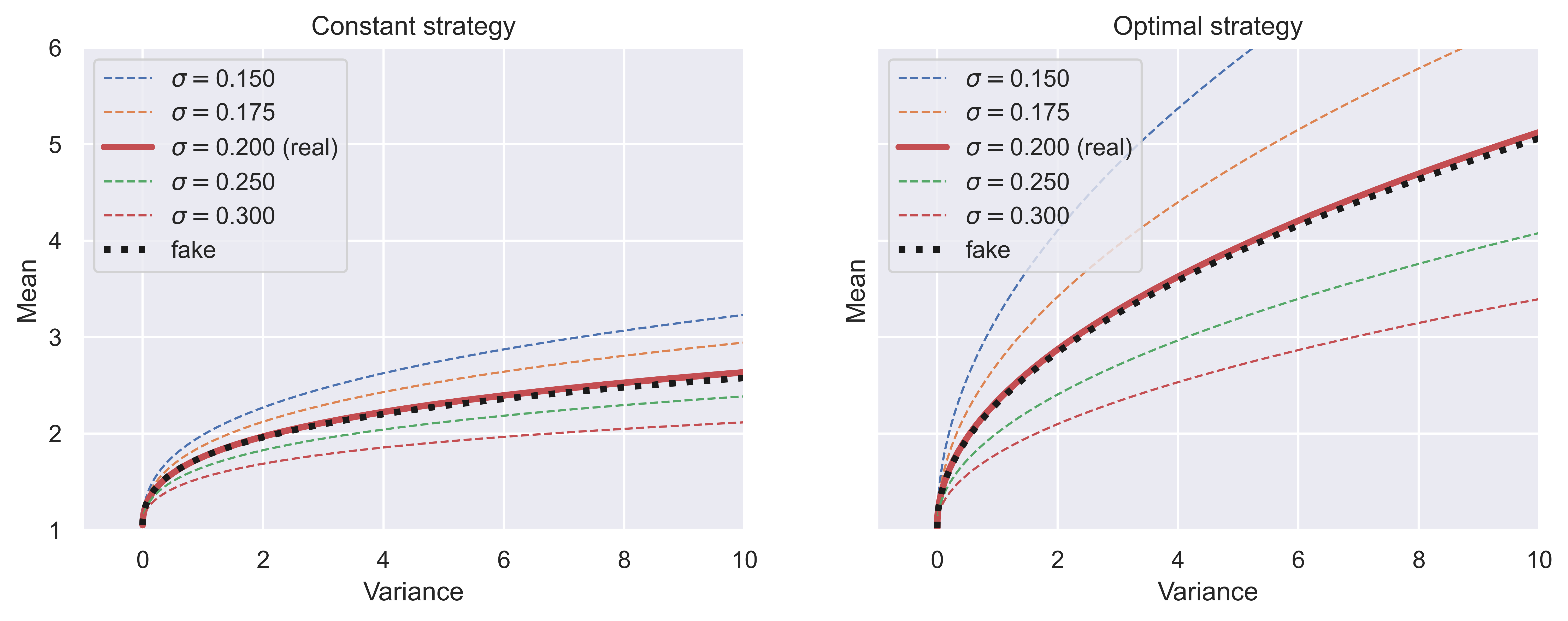}
\caption{Visualization of efficient mean-variance portfolio frontiers for real (red bold solid line), fake (black dotted lines) and control (dashed lines) paths. On the left-hand side, we show efficient frontiers under constant proportional trading strategies. On the right-hand side, we show efficient frontiers under optimal strategies.}
\label{fig:mean_variance_portfolio}
\end{figure}

\subsubsection*{Log-utility maximization}
Next we consider the log-utility maximization problem (Merton's problem \cite{Merton1975Oca}):
\begin{equation*}
    \calV(\mu) = \sup_{\theta} \E_{\mu}[\log(V^{\theta}_T)],
\end{equation*}
under the same setting as the mean-variance problem above. For both real and fake paths, we solve the log-utility maximization problem numerically among constant proportional trading strategies. We denote $v^{*}$ the theoretical optimal utility,  $\mu_\mathrm{real}$ the empirical measure of real paths with $50000$ samples,  $\mu_\mathrm{fake}$ the empirical measure of fake paths with $50000$ samples, $G^*$ the optimal strategy, $G_{\mathrm{real}}$ the optimal constant proportional trading strategy under $\mu_\mathrm{real}$, $G_{\mathrm{fake}}$ the optimal constant proportional trading strategy under $\mu_\mathrm{fake}$, and $V(\mu, G)$ the expected log-utility which arises from using the trading strategy $G$ in the market $\mu$. We compute $V(\mu_\mathrm{real}, G_{\mathrm{real}})$, $V(\mu_\mathrm{real}, G_{\mathrm{fake}})$, $V(\mu_\mathrm{fake}, G_{\mathrm{real}})$ and $V(\mu_\mathrm{fake}, G_{\mathrm{fake}})$, and compare them on the left-hand side of Figure~\ref{fig:log_utility}. To benchmark, we compute the numerically optimal expected log-utility of real paths with 1000 samples (same size as training data). Then we calculate the utilities with 200 random realizations and plot the histogram also on the left-hand side of Figure~\ref{fig:log_utility}. First we notice that $V(\mu_\mathrm{real}, G_{\mathrm{real}})$ is close to $v^{*}$. This justifies reliability of the numerical solver. Since $V(\mu_\mathrm{real}, G_{\mathrm{real}})$, $V(\mu_\mathrm{real}, G_{\mathrm{fake}})$, $V(\mu_\mathrm{fake}, G_{\mathrm{real}})$ and $V(\mu_\mathrm{fake}, G_{\mathrm{fake}})$ are all within the support of the histogram, they are considered relatively close to $v^{*}$. In particular, applying $G_{\mathrm{fake}}$ to $\mu_{\mathrm{real}}$ yields a very close estimate of the optimal utility. 

On the right-hand side of Figure~\ref{fig:log_utility}, we plot the curve of theoretical optimal log-utility vs volatility. We can interoperate $V(\mu_\mathrm{fake}, G_{\mathrm{fake}})$ as the optimal log-utility of Black-Scholes distribution with a different volatility level, which is very close to the real volatility level $\sigma=0.2$.

\begin{figure}[H]
\centering
\includegraphics[width=\textwidth]{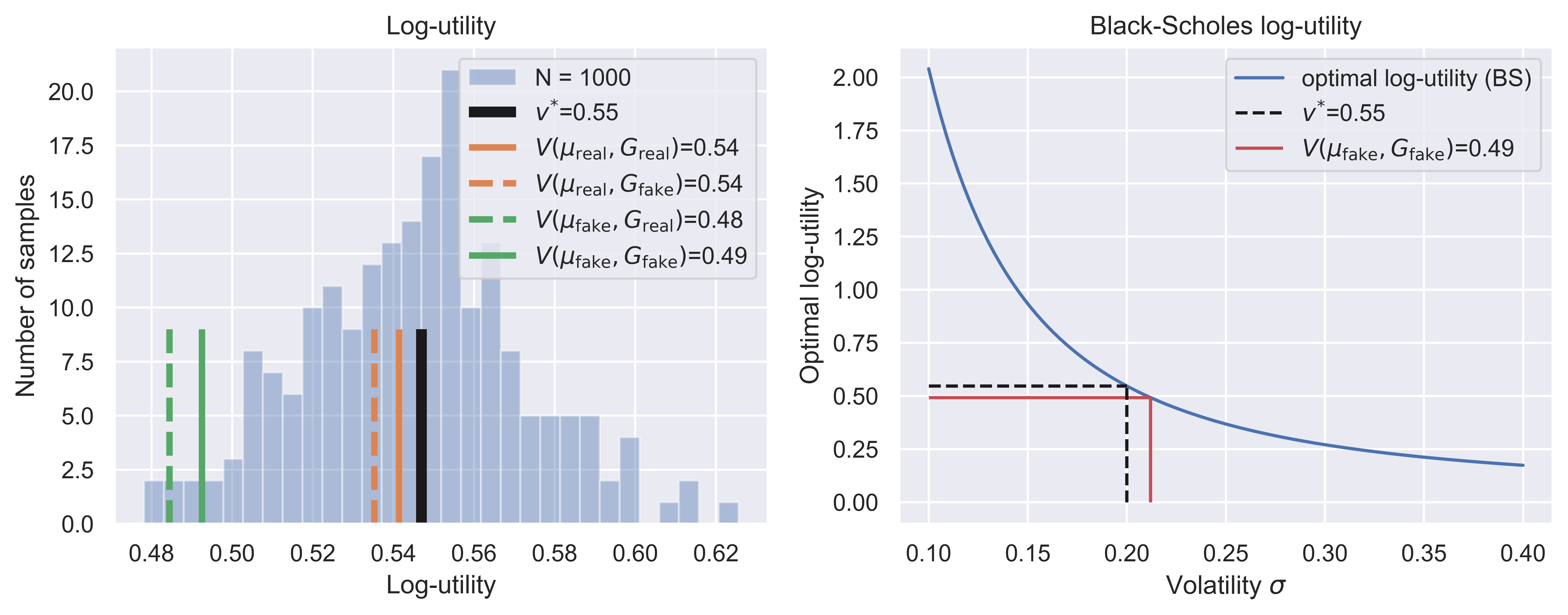}
\caption{On the left-hand side, we compare the expected log-utility of $V(\mu_\mathrm{real}, G_{\mathrm{real}})$, $V(\mu_\mathrm{real}, G_{\mathrm{fake}})$, $V(\mu_\mathrm{fake}, G_{\mathrm{real}})$ and $V(\mu_\mathrm{fake}, G_{\mathrm{fake}})$. To benchmark, the blue histogram showcases numerical optimal utilities values calculated with 1000 samples, using 200 different random seeds for the histogram. On the right-hand side, the blue curve illustrates theoretical optimal log-utility vs volatility.}
\label{fig:log_utility}
\end{figure}

\subsubsection*{Optimal Stopping problem}
Now we consider the optimal stopping problem for an American put option written on an underlying asset $S$ (given by either real, fake, or control paths). The stochastic optimization problem is given by
\begin{equation}
    \sup_{\tau\in \calT}\E[g(\tau, S_{\tau})],
\end{equation}
where $\calT$ is the set of stopping times and $g(t,S) = e^{-rt}\max(S - K)^{+}$, $r=0.1$, $S_0=100$, $K=100$. Control paths, as before, are discretized Black-Scholes paths different from real paths only in volatility. For real, fake and control paths, we compute the optimal stopping values with the deep optimal stopping solver introduced in \cite{Becker2019Dos}. We compute the optimal stopping values for real and fake paths, where each computation is repeated with 10 different random seeds, and compare them in Figure~\ref{fig:optimal_stopping_control}. The optimal stopping values are almost indistinguishable.
\begin{figure}[H]
\centering
\includegraphics[width=\textwidth]{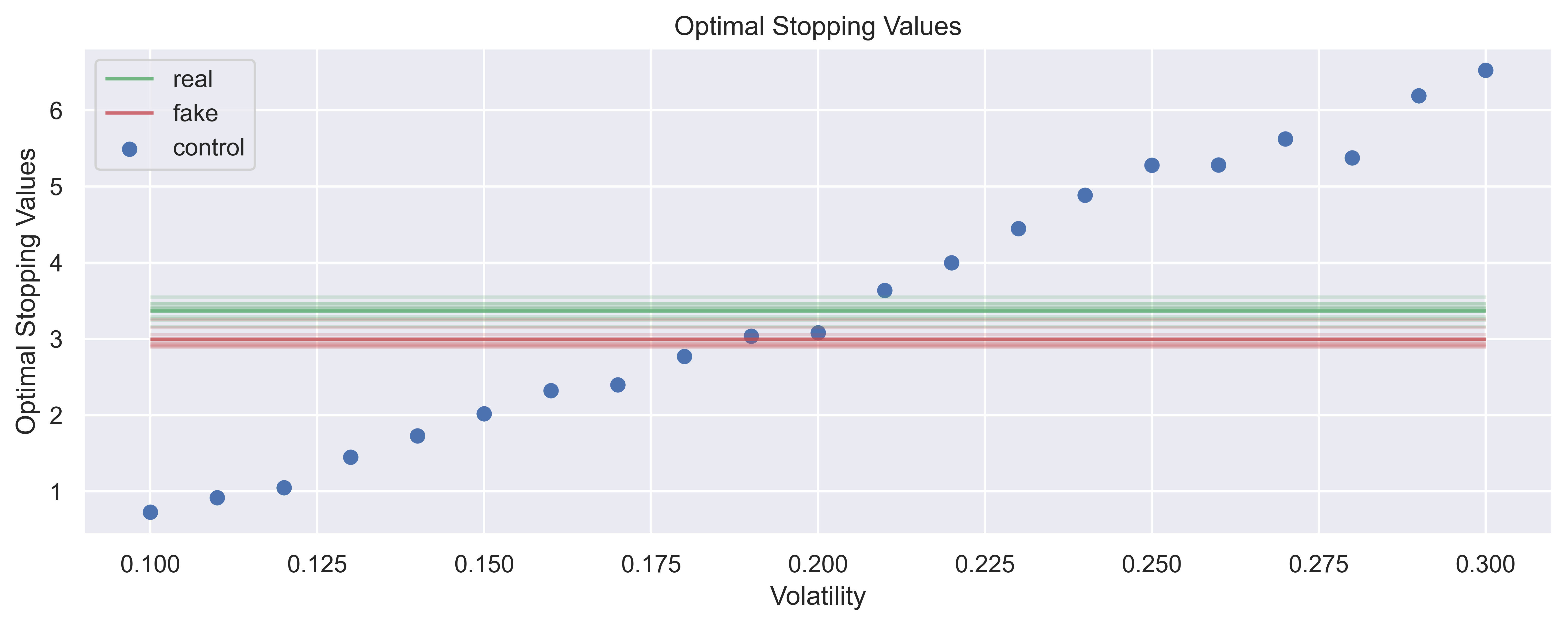}
\caption{Optimal stopping values under real (Black-Scholes model), fake (TC-VAE) and control distributions (varying volatility for the Black-Scholes model).}
\label{fig:optimal_stopping_control}
\end{figure}

\subsubsection*{Sliced adapted Wasserstein distance}
Ideally, we would like to compute the adapted Wasserstein distance between real and fake paths. However, the adapted Wasserstein distance is computationally heavy, and hence we instead evaluate a sliced version of it. Hereby, we draw inspiration from the growing literature on sliced Wasserstein distances (see e.g.~\cite{Deshpande2019Msw,Kolouri2019Gsw,Nietert2022Sra}). In general, slicing is a technique to reduce computational burden by considering low-dimensional projections of the distributions of interest. For comparing time series distributions, we interpret slicing in the sense that we only compute the adapted Wasserstein distance over subsets of time, and then average over those subsets. To this end, for a set $I \subset \{1, \dots, T\}$, denote by $\mu_I$ the marginal distribution of $\mu$ on the subset of times indexed by $I$, and by $|I|$ the number of elements in $I$. Given a distribution $\gamma$ 
over such subsets $I$, we define
$$
\mathcal{SAW}_{1}(\mu, \nu) := \int \mathcal{AW}_1(\mu_I, \nu_I) \,\gamma(dI).
$$
In practice, we use as $\gamma$ the uniform distribution over subsets of a certain size, called $n_{\mathrm{len}}$. The $\mathcal{AW}$ distances are evaluated using adapted empirical measure transformations (see \cite{AccHou2024, BackhoffVeraguas2022Epi,Hou2024Cot}) and then computed using backward induction (cf.~\cite{BackhoffVeraguas2017Cti,Pichler2022TnS}). To be precise, computation of adapted empirical measures requires clustering of the support points of the considered distributions. While \cite{BackhoffVeraguas2022Epi} introduced the adapted empirical measure using clusters based on a predefined grid, we instead use K-means clustering as in \cite{Benezet2024Lcd}, thus effectively adjusting the grids to the particular distributional shapes at hand. For the numerical implementation of the backward induction, we use the POT package \cite{Flamary2021PPO} for solving the inner optimal transport problems, i.e., for calculating each value in the dynamic programming formulation (see \cite[equation (5.1)]{BackhoffVeraguas2017Cti}). We believe that evaluating $\mathcal{SAW}_1$ gives a good indicator of temporal similarity between two measures while significantly reducing computational burden compared to $\mathcal{AW}_1$. However, one must be clear that certain information is lost through slicing, like complex long-range temporal dependencies. 

We configure the size of time slice $n_{\mathrm{len}} = 5$, the number of subsets $n_{\mathrm{slice}} = 100$, the number of samples $n_{\mathrm{sample}} = 500$, and the number of random seeds $n_{\mathrm{seed}} = 100$. Under this configuration, we compute $\mathcal{SAW}_{1}$ between real vs real paths (with different random seed), real vs fake paths, and real vs control paths. Control paths, as before, are discretized Black-Scholes paths different from real paths only in volatility, with $\sigma = 0.3$ instead of $\sigma = 0.2$ for real paths. We compare the average distance across $n_{\mathrm{seed}} = 100$ realizations and the standard deviation in Table~\ref{table:slicedadapted}. We find that $\mathcal{SAW}_{1}(\mu_{\mathrm{real}},\mu_{\mathrm{real}}')$ and $\mathcal{SAW}_{1}(\mu_{\mathrm{real}},\mu_{\mathrm{fake}})$ are almost indistinguishable, indicating that real and fake paths are relatively close under $\mathcal{SAW}_{1}$. Hereby, it is worth noting that the relatively high mean difference $\mathcal{SAW}_{1}(\mu_{\mathrm{real}},\mu_{\mathrm{real}}')$ can be explained by the fact that our models lie in a very high-dimensional space, and fine-grained high-dimensional similarity is difficult to obtain with only 500 samples. The fact that $\mathcal{SAW}_{1}(\mu_{\mathrm{real}},\mu_{\mathrm{real}}')$ is however relatively stable (low standard deviation) and much lower than $\mathcal{SAW}_{1}(\mu_{\mathrm{real}},\mu_{\mathrm{control}})$ indicates that there are certain low-dimensional features which most sample paths share. In fact, these features may be learned by the generator model as well, which could explain the similarly low values of $\mathcal{SAW}_{1}(\mu_{\mathrm{real}},\mu_{\mathrm{fake}})$.

\begin{table}[H]
	\begingroup
	\centering
	\begin{tabular}{l l l l}
		\toprule
		Description & Distance & Mean difference& Standard deviation\\\midrule
		real-real (different samples) & $\mathcal{SAW}_{1}(\mu_{\mathrm{real}},\mu_{\mathrm{real}}')$ & 0.367 & 0.032  \\
		real-fake & $\mathcal{SAW}_{1}(\mu_{\mathrm{real}},\mu_{\mathrm{fake}})$ & 0.382 & 0.028  \\
		real-control & $\mathcal{SAW}_{1}(\mu_{\mathrm{real}},\mu_{\mathrm{control}})$ & 0.670 & 0.066 \\
        \bottomrule
	\end{tabular}\\\vspace{1.5mm}
	\endgroup
     \caption{Sliced adapted Wasserstein distances between different measures.}
    \label{table:slicedadapted}
\end{table}

\subsubsection{Heston model}
\label{sec:hest}
Next, we consider a more complicated temporal dynamic given by the Heston model: $(S^{\mathrm{H}}_t, V^{\mathrm{H}}_t)_{t\geq 0}$ s.t.
\begin{align*}
    \frac{dS^{\mathrm{H}}_t}{S^{\mathrm{H}}_t} &= \mu dt + \sqrt{V^{\mathrm{H}}_t}dW^S_t,\quad S^{\mathrm{H}}_0 = 1,\\
    dV^{\mathrm{H}}_t &= \kappa(\theta - V^{\mathrm{H}}_t) dt + \xi\sqrt{V^{\mathrm{H}}_t}dW_t^V, \quad V^{\mathrm{H}}_0 = \theta,
\end{align*}
where $(W^S_t, W^V_t)_{t\geq 0}$ are Wiener processes with correlation $\rho= -0.9$, $\mu = 0.02$, $\kappa = 1$, $\theta = 0.2$, $\xi = 0.5$. We aim at learning the distribution of $(S^{\mathrm{H}}_t)_{t\geq 0}$. Under the same time discretization and sample size as in the Black-Scholes case of Section~\ref{sect.bs}, we conduct the same tests (omitting the visualization of paths and marginal distribution for the sake of space) and observe a similar performance. For this, see sliced Wasserstein distance, Gaussian kernel MMD, and signature MMD compared in Figure~\ref{fig:heston_mmd_swd}; optimal stopping values compared in Figure~\ref{fig:heston_optimal_stopping}; sliced adapted Wasserstein distance compared in Table~\ref{table:heston_slicedadapted}.

\begin{figure}[H]
    \centering
    \includegraphics[width=\linewidth]{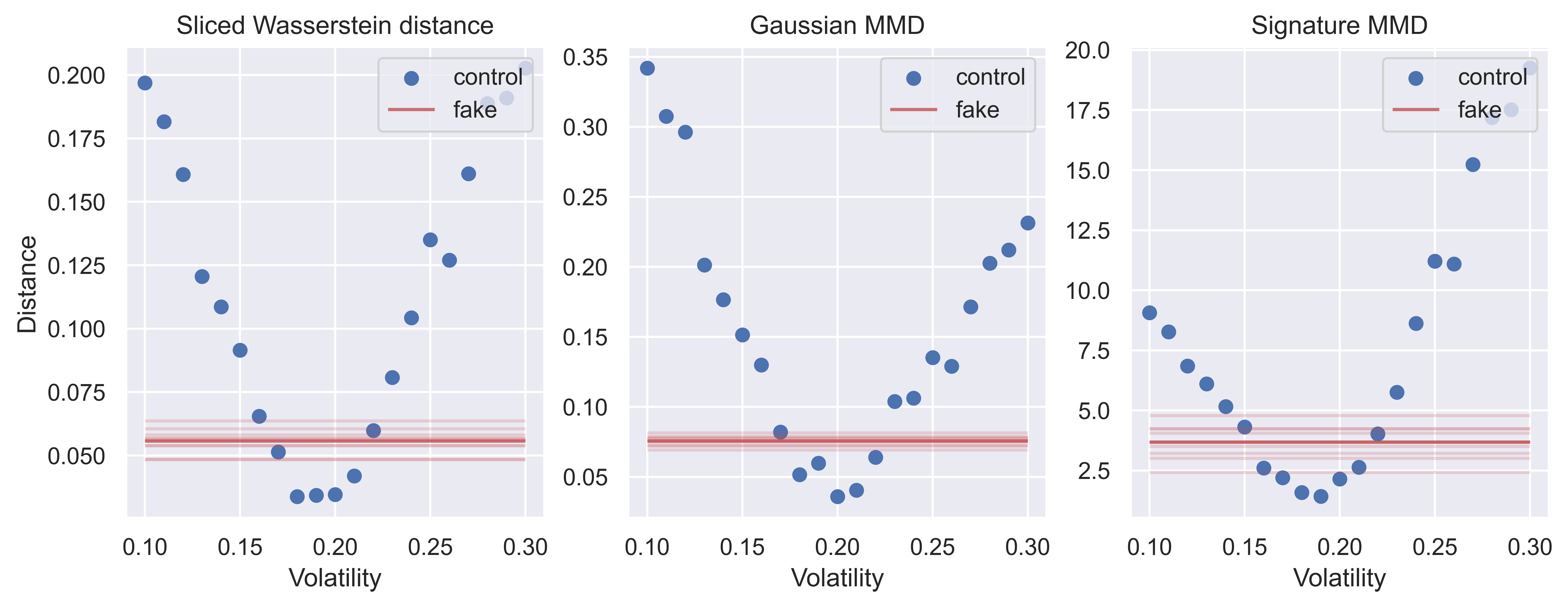}
    \caption{From left to right, we visualize the sliced Wasserstein distance, Gaussian MMD, and signature MMD. The red lines illustrate distances between real paths of the Heston model and fake paths generated from the TC-VAE model (each line is a different random seed). The blue dots show the distances between real paths and control paths. Control paths are discretized Heston paths different from real paths only in $\theta$.
}
    \label{fig:heston_mmd_swd}
\end{figure}

\begin{figure}[H]
\centering
\includegraphics[width=\textwidth]{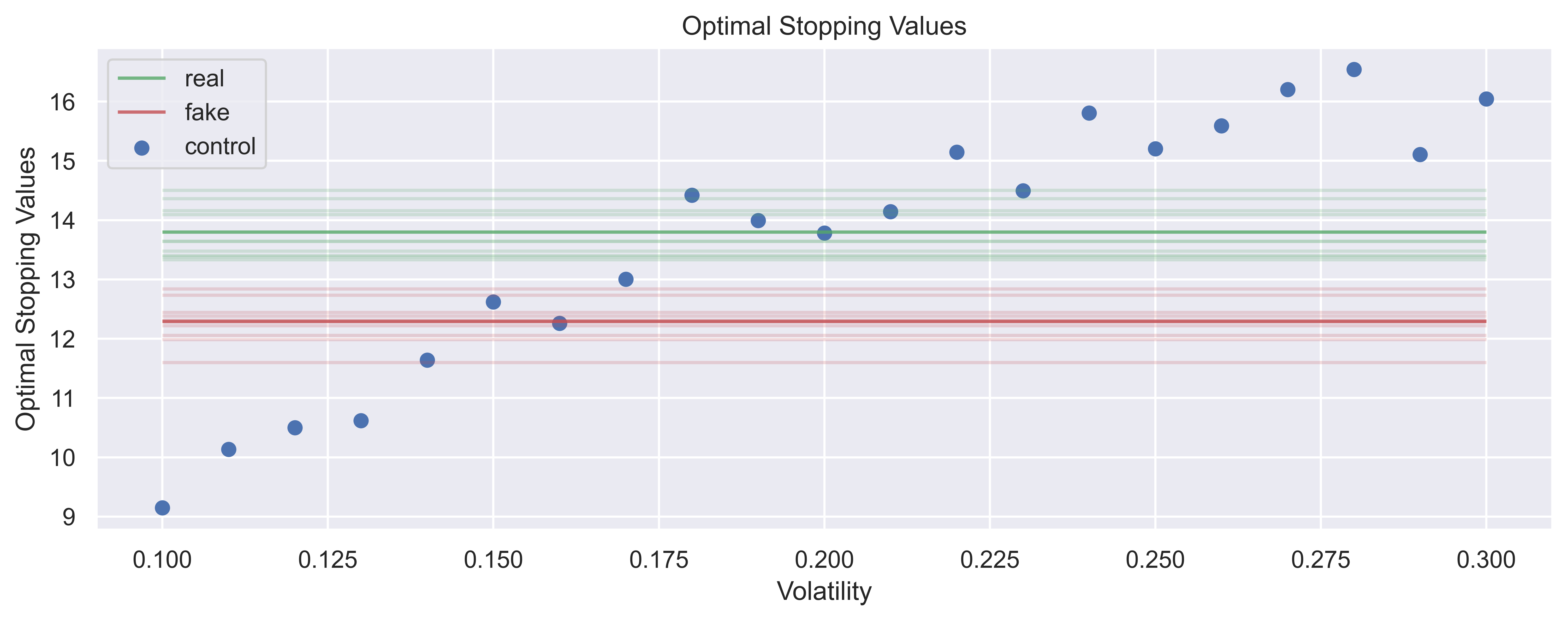}
\caption{Optimal stopping values under real, fake and control distributions in case of synthetic data using the Heston model. Control paths are discretized Heston paths different from real paths only in $\theta$.}
\label{fig:heston_optimal_stopping}
\end{figure}

\begin{table}[H]
\begingroup
\centering
\begin{tabular}{l l l l}
    \toprule
    Description & Formula & Mean difference& Standard deviation\\\midrule
    real-real (different samples) & $\mathcal{SAW}_{1}(\mu_{\mathrm{real}},\mu_{\mathrm{real}}')$ & 0.505 & 0.027  \\
    real-fake & $\mathcal{SAW}_{1}(\mu_{\mathrm{real}},\mu_{\mathrm{fake}})$ & 0.549 & 0.039  \\
    real-control & $\mathcal{SAW}_{1}(\mu_{\mathrm{real}},\mu_{\mathrm{control}})$ & 0.700 & 0.052 \\
    \bottomrule
\end{tabular}\\\vspace{1.5mm}
\endgroup
\noindent
    \caption{Sliced adapted Wasserstein distances between different measures arising from the Heston model. Control paths are discretized Heston paths different from real paths only in $\theta$ ($0.15$ vs $0.2$).}
\label{table:heston_slicedadapted}
\end{table}

\subsubsection{Path dependent volatility model}
\label{sect.pdv}
In a financial market, we only observe a single realization of a path, such as stock prices, rather than i.i.d. paths. Although one can use rolling windows to sample many sub-paths and assume them to be i.i.d. samples, this clearly leads to several risks. First of all, the sub-paths are in fact not independent and this causes severe over-fitting. Even worse, when the observed path is short, in order to extract more sub-paths for training, the rolling windows greatly overlap and causes even higher correlation. Although non-overlapping windows can alleviate correlation, this requires much longer observed paths, which is sometimes not possible. Even when this is possible, this exposes the model to distributional shift over time, which is related to another issue: non-stationarity. The sub-paths are also not identically distributed if the observed path is not stationary. This is common in financial data where the distribution shifts swiftly over time. Thus, the sub-paths might only be correlated samples coming from a different distribution. In the end, we are learning an average distribution over time, which is meaningless for forecasts about possible future evolution. To tackle this issue, we deploy the temporal nature of financial time series by generating future paths from real historical paths. To do so, we further develop a conditional version of TC-VAE. The only difference compared to TC-VAE is that we concatenate the latent variable with an additional conditional variable, see Figure~\ref{fig:cvae}. 
\begin{figure}[H]
\centering
\includegraphics[width=0.9\textwidth]{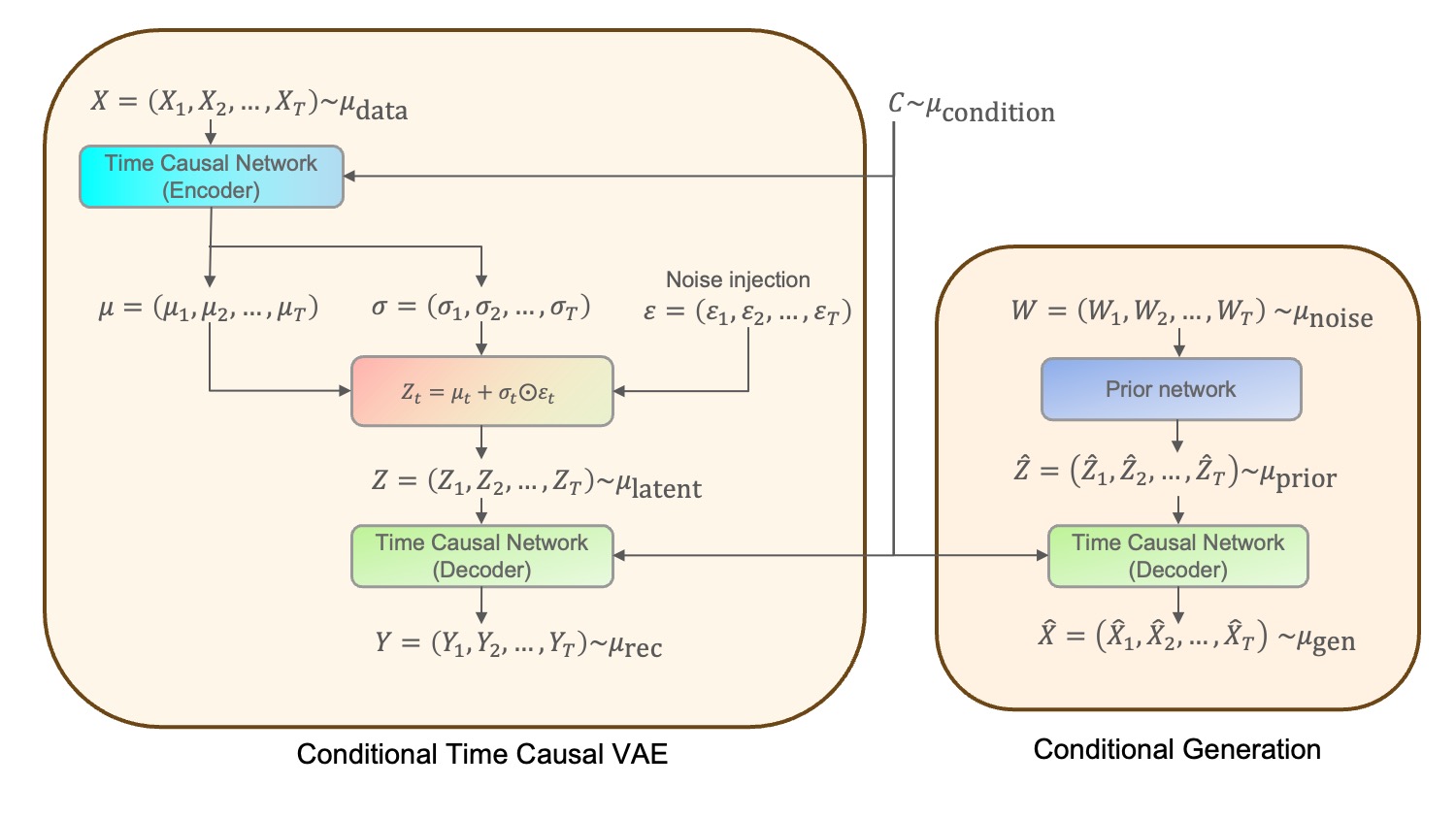}
\caption{Conditional time-causal variational autoencoder and generation}
\label{fig:cvae}
\end{figure}
With the \emph{conditional TC-VAE}, we can generate the distribution of future paths conditional on historical paths. In financial modelling, path dependent models have shown to be able to successfully capture the price dynamics. For example, consider the path dependent volatility model where the prices $(S_{t})_{t\geq 0}$ satisfy 
\begin{equation*}
    dS_{t} = S_t\sigma(S_{\leq t})dW_t,
\end{equation*}
where $\sigma$ is measurable and $(W_t)_{t\geq 0}$ is a Wiener process. Given the time-homogeneous dynamic, for every $\tau>0$, the law of $S_{[t,t+\tau]}$ conditional on $\sigma(S_{\leq t})$ is the same for all $t\geq 0$. As an example, we consider the 4-factor Markovian path dependent volatility model (PDV4) introduced in \cite{Guyon2023Vim}, where the volatility function is constructed by exponential kernels to produce the Markovian model $(S^{\mathrm{PDV}}_{t})_{t\geq 0}$ s.t. 
\begin{equation*}
    \left\{
    \begin{aligned}
        \frac{dS^{\mathrm{PDV}}_t}{S^{\mathrm{PDV}}_t} &= \mu dt + \sigma_t dW_t,\quad \sigma_t = \sigma(R_{1,t},R_{2,t}),\quad \sigma(R_{1},R_{2}) = \beta_0 + \beta_1 R_1 + \beta_2 \sqrt{R_2},\\
        R_{1,t} &= (1 - \theta_1)R_{1,1,t} + \theta_1 R_{1,2,t}, \quad R_{2,t} = (1 - \theta_2)R_{2,1,t} + \theta_2 R_{2,2,t},\\
        dR_{1,j,t} &= \lambda_{1,j} \Big(\sigma_t dW_t - R_{1,j,t}dt\Big), \quad dR_{2,j,t} = \lambda_{2,j} \Big(\sigma_t^2 dW_t - R_{1,j,t}\Big), j=1,2
    \end{aligned}\right.
\end{equation*}
where $\mu = 0.1, \beta_0 = 0.04, \beta_1 = -0.13, \beta_2 = 0.65, \lambda_{1,1} = 55, \lambda_{1,2} = 10, \theta_1 = 0.25, \lambda_{2,1} = 20, \lambda_{2,2} = 3, \theta_2 = 0.5$. 
PDV4 captures important stylized facts of volatility, produces very realistic price and volatility paths (see Figure~\ref{fig:pdv4_sp500}), and jointly fits SPX and VIX smiles remarkably well \cite{Guyon2023Vim}.
\begin{figure}[H]
    \centering
    \includegraphics[width=0.99\textwidth]{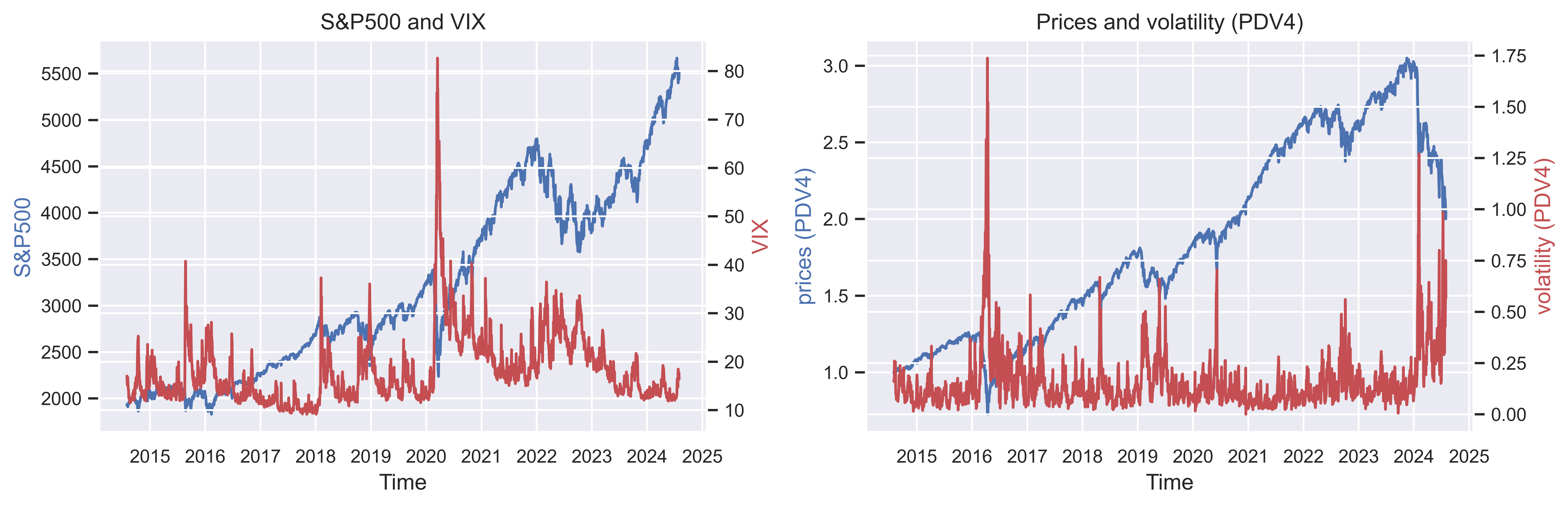}
    \caption{On the left, we plot daily S\&P500 (blue) and VIX (red) from 2014-08-01 to 2024-08-01. On the right, we plot prices (blue) and volatility (red) of the 4-factor Markovian path dependent volatility model.}
    \label{fig:pdv4_sp500}
\end{figure}
Since we work in discrete time, we let $dt = 1/365$, $N_T = 60$ to model daily prices. We choose $N_{\mathrm{sample}}=2560$ to be the number of samples in market data. Let $S_{t=1}^{N}$ be a discretized price path sampled from the PDV4 model. We extract the latest sub-paths $S^{(i)} = S_{i: i+N_T}$, $i \in \calI = \{N-N_T-N_\mathrm{sample},\dots,N-N_T\}$, and normalize them by dividing each sub-path by its starting price. This gives the return paths $X^{(i)} = S_{i: i+N_T} / S_{i}$, $i \in \calI$, which are our real paths. 
We denote the weighted historical volatility by
\[
\Sigma^{(i)} = \sqrt{\sum_{j\leq i} K_2(i-j)(\frac{S_{j} - S_{j-1}}{S_{j-1}})^2},\quad i \in \calI,
\]
where $K_2(k) = Z^{-1}_{\alpha, \delta}(k + \delta)^{-\alpha}$, $\alpha > 1$, $\delta > 0$ and $Z^{-1}_{\alpha, \delta}$ is chosen s.t. $\sum_{k=0}^{\infty}K_2(k) = 1$. 
Given the observed sample $(X^{(i)}, \Sigma^{(i)}), i\in\calI$, we apply the conditional TC-VAE to learn the distribution of future returns given the  weighted historical volatility. 

First, we visualize the fake paths under different conditions, see Figure~\ref{fig:pdv_con_gen}. The generator indeed generates different distributions conditional on different conditions. Moreover, the generated paths show gain/loss asymmetry and volatility clustering, which are stylized facts of financial time series \cite{Cont2001Epo}.
\begin{figure}[H]
    \centering
    \includegraphics[width=\textwidth]{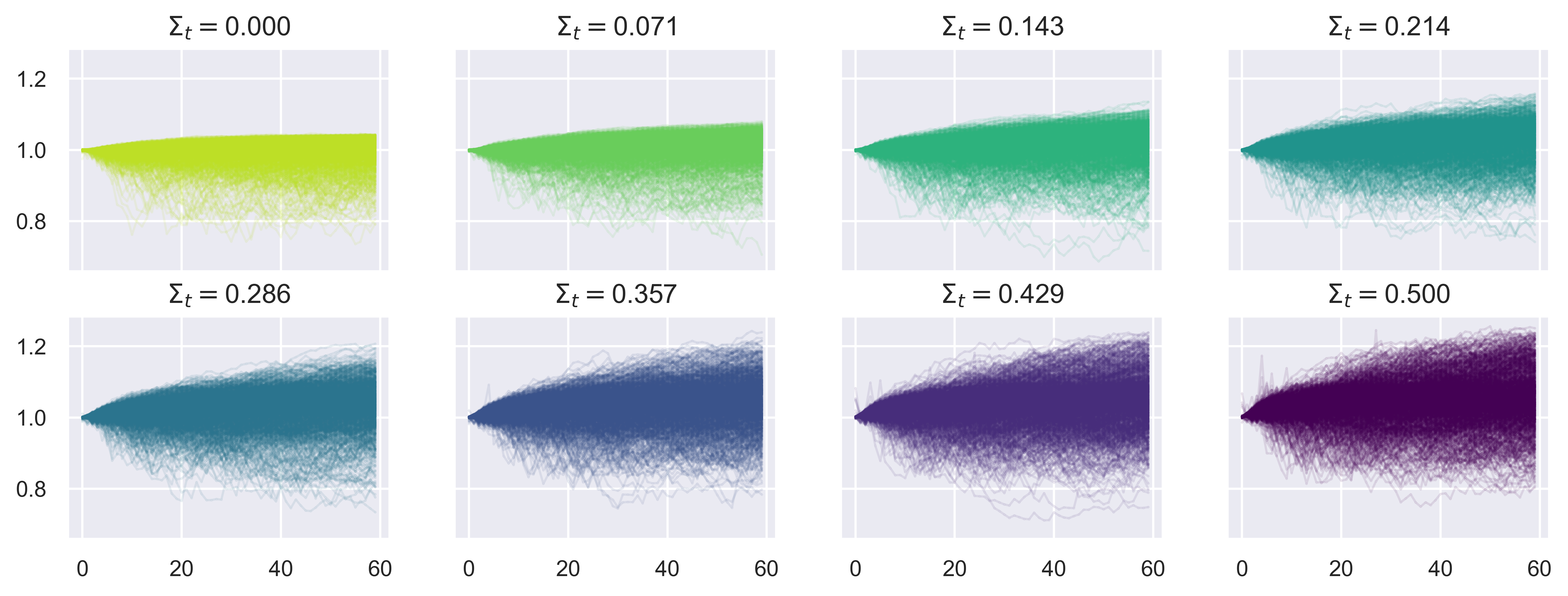}
    \caption{Visualization of generated returns conditional on weighted historical volatility}
    \label{fig:pdv_con_gen}
\end{figure}
Then we compare fake paths and true paths distributions conditional on the same history path. To benchmark, we also sample a Black-Scholes distribution with drift and volatility estimated from the past $100$ time steps. We quantitatively evaluate the conditional distributions under sliced Wasserstein distance, Gaussian MMD, signature MMD, and sliced adapted Wasserstein distance. For each historical path, we generate real, fake and control paths and compute distances between real vs real paths (with different random seed), real vs fake paths, and real vs control paths. For control paths, we use Black-Scholes paths with drift and volatility estimated from the historical paths, which serves as a benchmark. In Figure~\ref{fig:pdv_con_gen_com_swd_mmd_sawd}, we compare sliced Wasserstein distance, Gaussian MMD, signature MMD, and sliced adapted Wasserstein distance.

\begin{figure}[H]
    \centering
    \includegraphics[width=\textwidth]{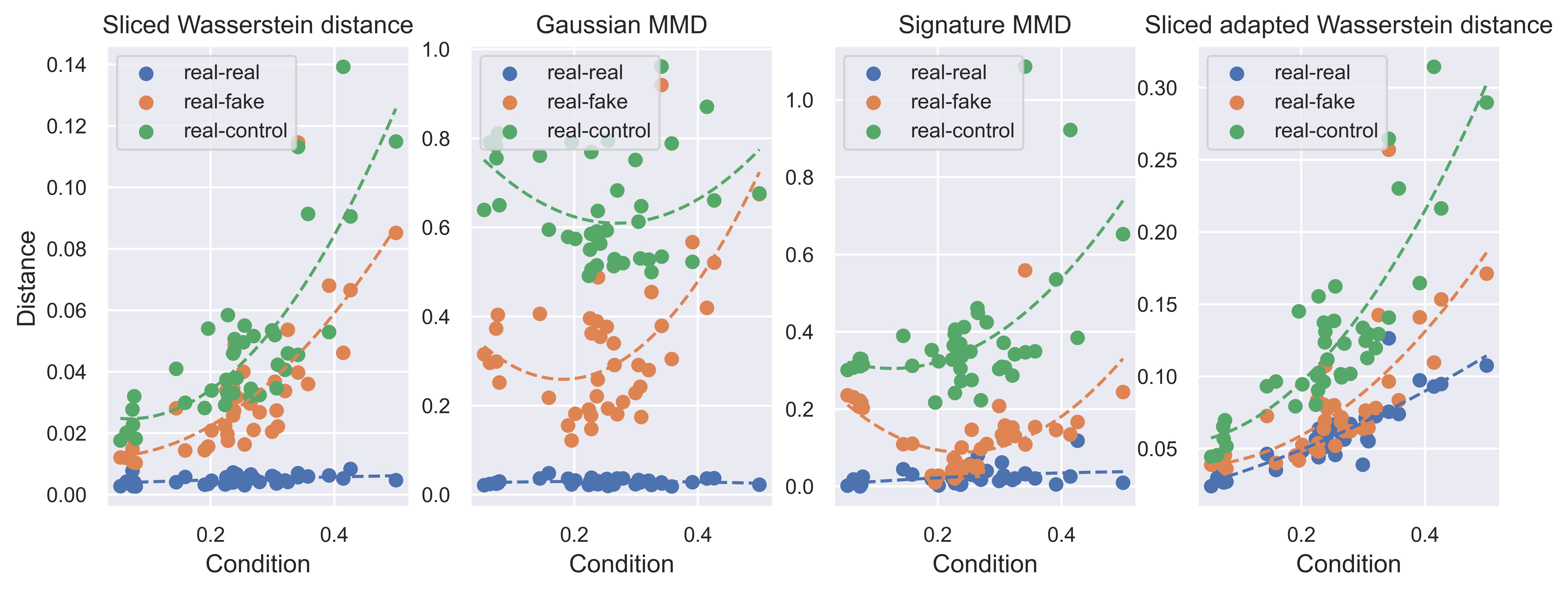}
    \caption{From left to right, we visualize the sliced Wasserstein distance, Gaussian MMD, signature MMD, sliced adapted Wasserstein distance. The dots are the distances between paths and the dash lines are quadratic polynomials fitted to the distances. 
    }
    \label{fig:pdv_con_gen_com_swd_mmd_sawd}
\end{figure}
Notably, the path generation is not constrained by the length of training paths, which means that we can extend a path as long as desired. We iterate the following two steps: 1) extending the path by conditional generation; 2) calculating conditions of the extended path. See Figure~\ref{fig:path_extend} for path extension of $600 = 10*N_T$ time steps after $10$ iterative path extensions. The generation shows long time stability without blowing up or vanishing.
\begin{figure}[H]
    \centering
    \includegraphics[width=\textwidth]{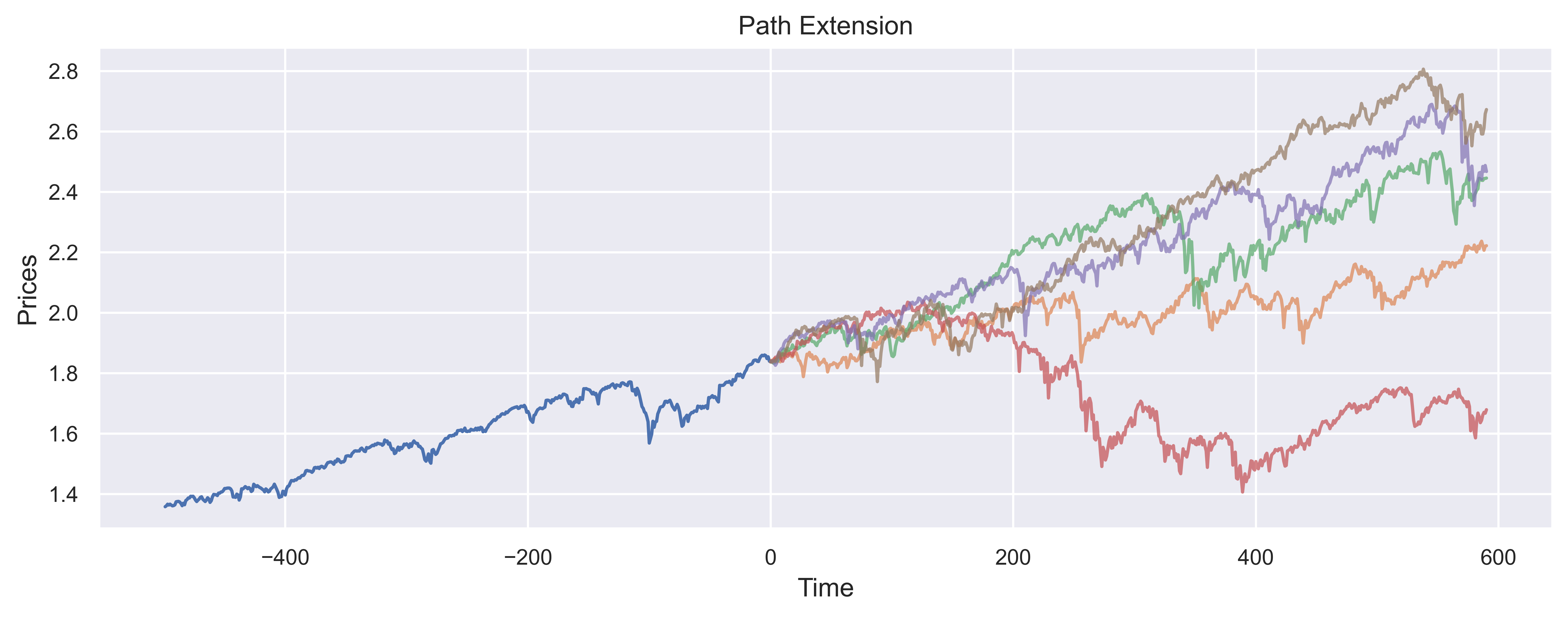}
    \caption{Extending a real path under the PDV4 model by generated conditional paths using TC-VAE.}
    \label{fig:path_extend}
\end{figure}

\subsection{Market data}
\subsubsection{S\&P500 and VIX}
\label{sect.market}
Encouraged by the promising conditional generation from the path-dependent model considered above, we now apply our generator to data from S\&P500 and VIX. We take daily S\&P500 and VIX from 2014-08-01 to 2024-08-01, for a total of $N = 2516$ trading days; see Figure~\ref{fig:pdv4_sp500}. We denote by $(S_t)_{t=1}^{N}$ the 
S\&P500 path and by $(V_t)_{t=1}^{N}$ the VIX path. As before, we consider subpaths of length $N_T = 60$ and extract the sub-paths $S^{(i)} = S_{i: i+N_T}$, $i \in \calI = \{1,\dots,N-N_T\}$ and normalize them by dividing each sub-path by its starting price. This gives the return paths $X^{(i)} = S_{i: i+N_T} / S_{i}$, $i \in \calI$, which are our real paths. As before, conditional on the VIX, we apply the conditional TC-VAE to learn S\&P500 prices in the future. With the data and condition pair $(X^{(i)}, \Sigma^{(i)})$, we apply the conditional TC-VAE to learn the distribution of $(X^{(i)})_{i\in\calI}$ given $(\Sigma^{(i)})_{i\in\calI}$. 

First, we visualize the fake paths under different conditions, see Figure~\ref{fig:sp500_con_gen}. The conditional generation shows skewness. Moreover, the volatility of generated paths is increasing with the VIX, which is in line with the financial interpretation of the condition. 
\begin{figure}[H]
    \centering
    \includegraphics[width=\textwidth]{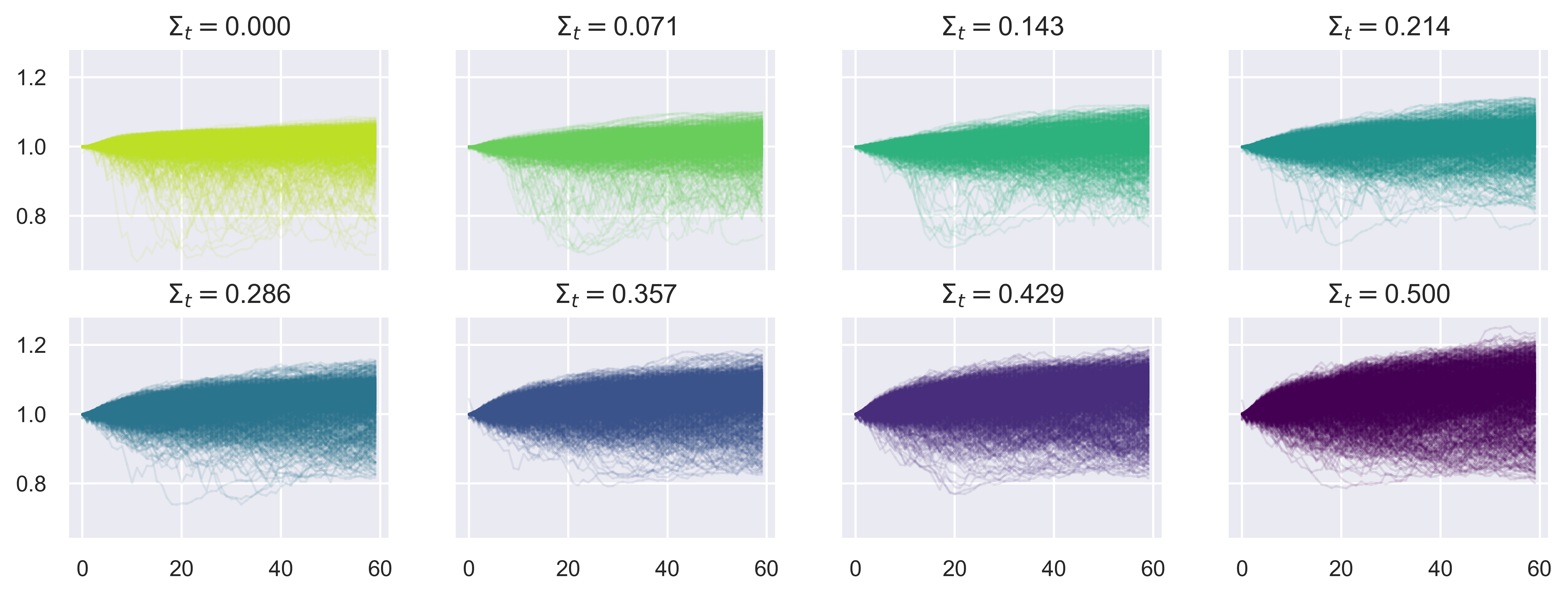}
    \caption{Visualization of generated returns conditional on VIX.}
    \label{fig:sp500_con_gen}
\end{figure}

VIX can be well estimated by the weighted historical volatility, see \cite{Guyon2023Vim} for a through analysis. Thus, similar to the PDV4 case, we can extend the path as long as desired. See Figure~\ref{fig:path_extend_sp500} for path extension of $600 = 10*N_T$ time steps after $10$ iterative path extensions.
\begin{figure}[H]
    \centering
    \includegraphics[width=\textwidth]{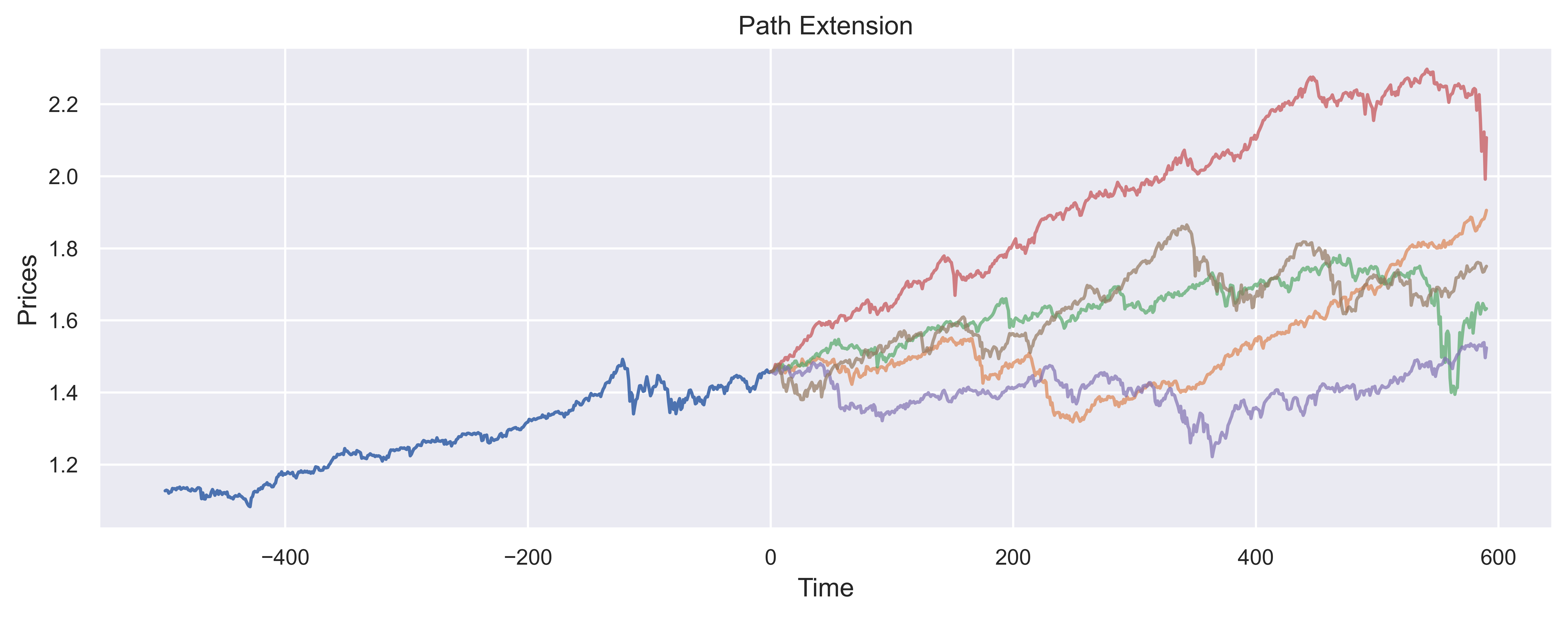}
    \caption{Path extension of normalized S\&P500 prices using TC-VAE.}
    \label{fig:path_extend_sp500}
\end{figure}
Finally, we generate a long fake path by path extension, and compare it with S\&P 500 prices in terms of stylized facts of financial time series; see \cite{Cont2001Epo}. This includes: 1) heavy tail of returns, 2) volatility clustering, 3) zero auto-correlation of returns, 4) short-time auto-correlation of square returns, 5) long-time auto-correlation of absolute returns, and 6) negative skewness of returns. Returns of both S\&P500 prices and fake prices display power-law or Pareto-like distribution; see Figure~\ref{fig:logreturn_volcluster}. The high-volatility returns of both S\&P500 prices and fake prices tend to cluster, which is known as volatility clustering.
\begin{figure}[H]
    \centering
    \includegraphics[width=\textwidth]{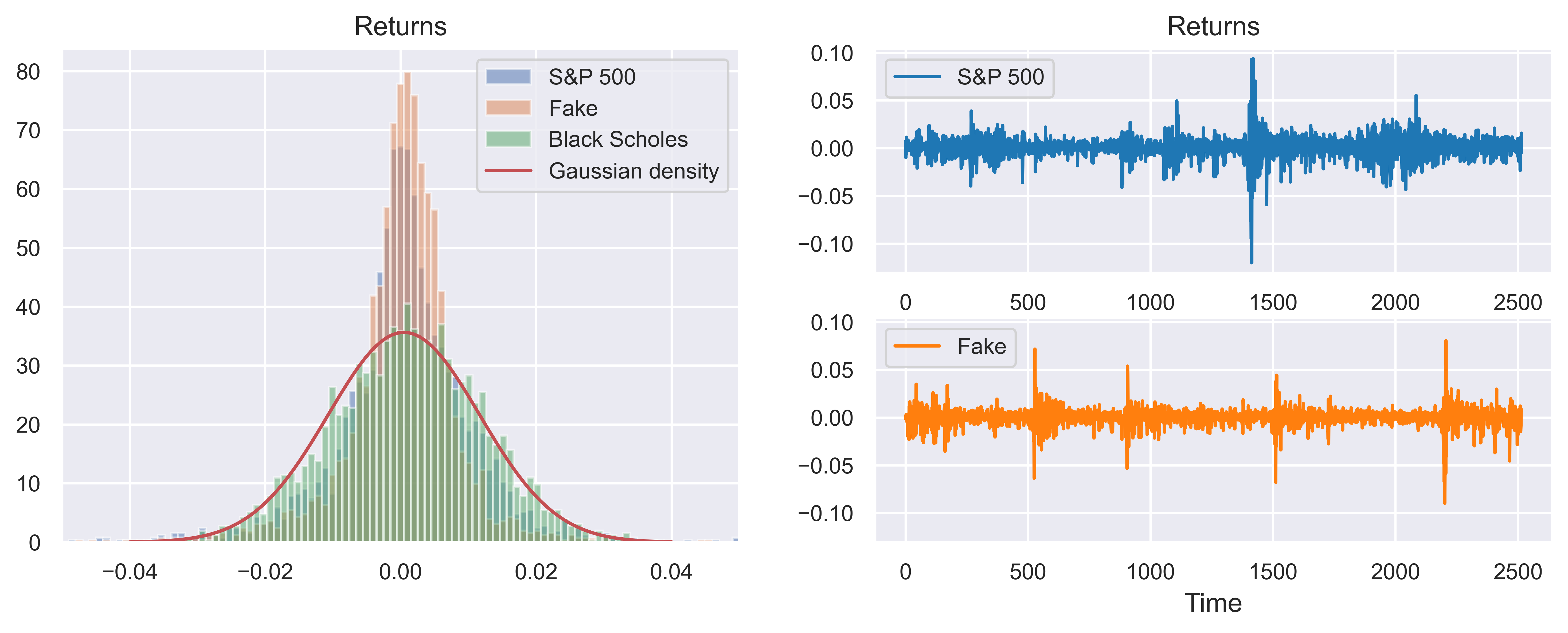}
    \caption{On the left, we visualize returns histogram of S\&P 500 (blue), fake prices (orange), Black Scholes prices (green), and the Gaussian density (red). On the right we plot returns of S\&P 500 (blue) and fake prices (orange) along the time horizon.}
    \label{fig:logreturn_volcluster}
\end{figure}
Furthermore, we inspect the auto-correlation of returns. Returns of S\&P500 prices and fake prices both show no correlation in returns, short time correlation in square returns, and long time correlation in absolute returns; see Figure~\ref{fig:autocorrelation}.
\begin{figure}[H]
    \centering
    \includegraphics[width=\textwidth]{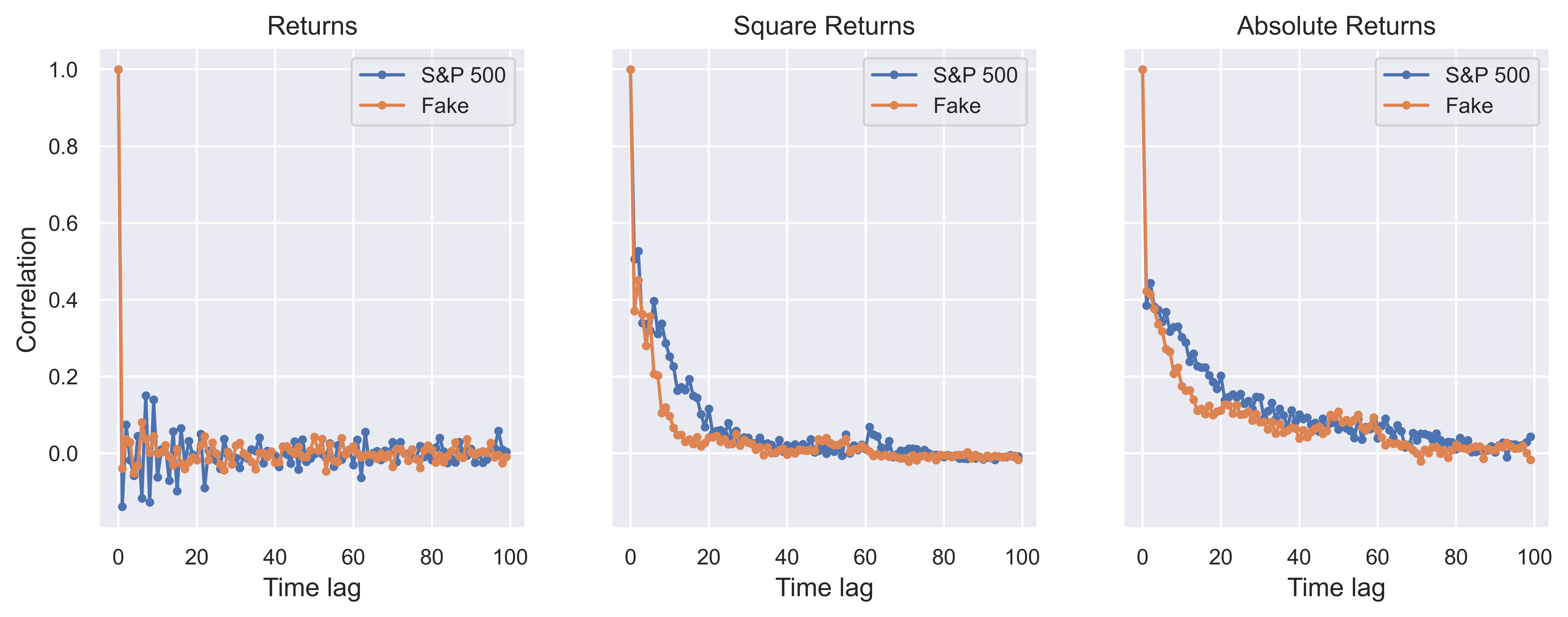}
    \caption{From left to right, we visualize the auto-correlation of returns, square returns, and absolute returns for both S\&P 500 prices (blue) and fake prices (orange).}
    \label{fig:autocorrelation}
\end{figure}
Lastly, we compare the skewness and kurtosis of returns for both S\&P500 from 2014-08-01 to 2024-08-01 and $1000$ fake paths; see Figure~\ref{fig:skew_kurtosis}. Overall, S\&P 500 prices and fake prices are close in skewness and kurtosis of returns.
\begin{figure}[H]
    \centering
    \includegraphics[width=\textwidth]{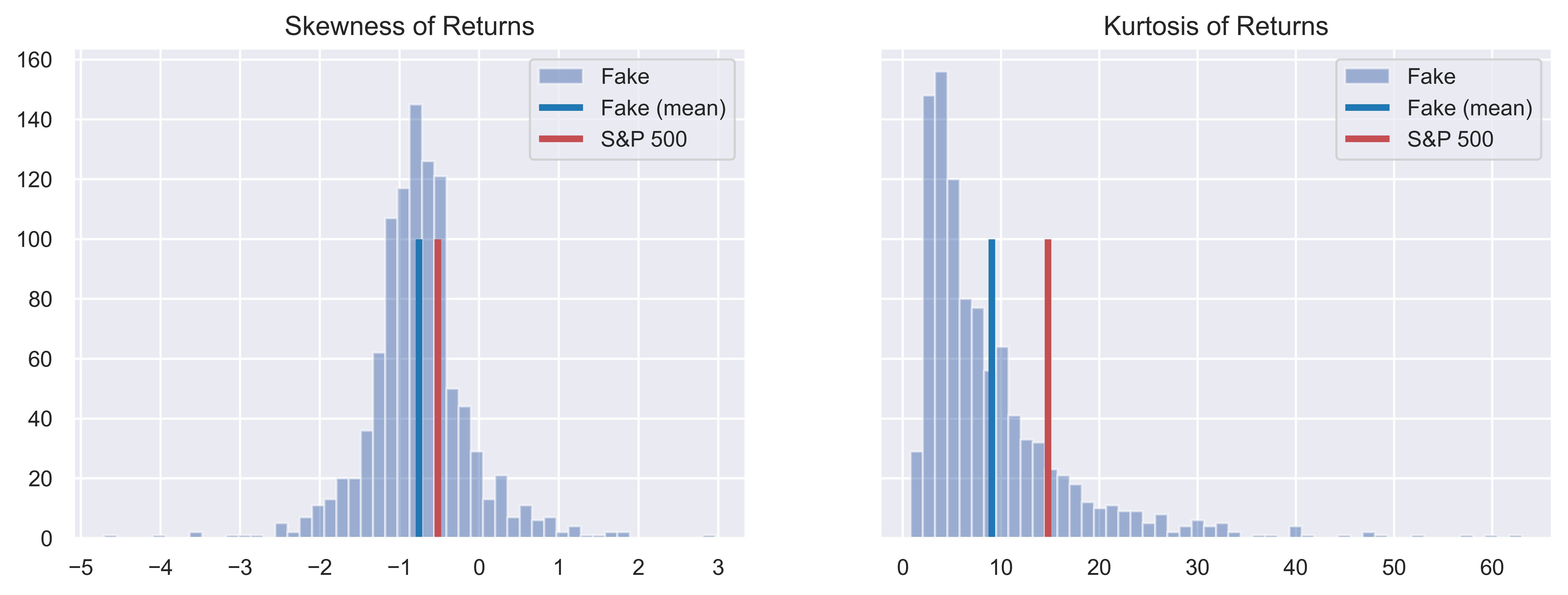}
    \caption{On the left, we plot the histogram of skewness of returns for $1000$ fake paths, their mean (blue), and the skewness of returns for S\&P 500. On the right, we plot the histogram of kurtosis of returns for $1000$ fake paths, their mean (blue), and the kurtosis of returns for S\&P 500.}
    \label{fig:skew_kurtosis}
\end{figure}

\printbibliography

@Article{BackhoffVeraguas2020AWd,
  author    = {Backhoff-Veraguas, Julio and Bartl, Daniel and Beiglb{\"o}ck, Mathias and Eder, Manu},
  journal   = {Finance and Stochastics},
  title     = {Adapted Wasserstein distances and stability in mathematical finance},
  year      = {2020},
  pages     = {601--632},
  volume    = {24},
  publisher = {Springer},
}

@Article{Pflug2012Adf,
  author    = {Pflug, Georg Ch and Pichler, Alois},
  journal   = {SIAM Journal on Optimization},
  title     = {A distance for multistage stochastic optimization models},
  year      = {2012},
  number    = {1},
  pages     = {1--23},
  volume    = {22},
  publisher = {SIAM},
}

@Article{Gretton2012Akt,
  author    = {Gretton, Arthur and Borgwardt, Karsten M and Rasch, Malte J and Sch{\"o}lkopf, Bernhard and Smola, Alexander},
  journal   = {The Journal of Machine Learning Research},
  title     = {A kernel two-sample test},
  year      = {2012},
  number    = {1},
  pages     = {723--773},
  volume    = {13},
  publisher = {JMLR. org},
}

@Article{Pammer2024Ano,
  author    = {Pammer, Gudmund},
  journal   = {Electronic Communications in Probability},
  title     = {A note on the adapted weak topology in discrete time},
  year      = {2024},
  pages     = {1--13},
  volume    = {29},
  publisher = {The Institute of Mathematical Statistics and the Bernoulli Society},
}

@Article{Kingma2014Aev,
  author  = {Kingma, Diederik P and Welling, Max},
  journal = {2nd International Conference on Learning Representations, ICLR 2014},
  title   = {Auto-encoding variational bayes},
  year    = {2014},
}

@InProceedings{Higgins2016bvL,
  author    = {Higgins, Irina and Matthey, Loic and Pal, Arka and Burgess, Christopher and Glorot, Xavier and Botvinick, Matthew and Mohamed, Shakir and Lerchner, Alexander},
  booktitle = {International conference on learning representations},
  title     = {beta-vae: Learning basic visual concepts with a constrained variational framework},
  year      = {2016},
}

@Article{Acciaio2020Cot,
  author    = {Acciaio, Beatrice and Backhoff-Veraguas, Julio and Zalashko, Anastasiia},
  journal   = {Stochastic Processes and their Applications},
  title     = {Causal optimal transport and its links to enlargement of filtrations and continuous-time stochastic optimization},
  year      = {2020},
  number    = {5},
  pages     = {2918--2953},
  volume    = {130},
  publisher = {Elsevier},
}

@Article{Lassalle2018Ctp,
  author    = {Lassalle, R{\'e}mi},
  journal   = {Stochastic Processes and their Applications},
  title     = {Causal transference plans and their Monge-Kantorovich problems},
  year      = {2018},
  number    = {3},
  pages     = {452-484},
  volume    = {36},
  publisher = {Elsevier},
}

@Article{BackhoffVeraguas2017Cti,
  author    = {Backhoff-Veraguas, Julio and Beiglbock, Mathias and Lin, Yiqing and Zalashko, Anastasiia},
  journal   = {SIAM Journal on Optimization},
  title     = {Causal transport in discrete time and applications},
  year      = {2017},
  number    = {4},
  pages     = {2528--2562},
  volume    = {27},
  publisher = {SIAM},
}

@Article{Bremaud1978Cof,
  author    = {Br{\'e}maud, Pierre and Yor, Marc},
  journal   = {Zeitschrift f{\"u}r Wahrscheinlichkeitstheorie und verwandte Gebiete},
  title     = {Changes of filtrations and of probability measures},
  year      = {1978},
  number    = {4},
  pages     = {269--295},
  volume    = {45},
  publisher = {Springer},
}

@Article{Eckstein2022Cmf,
  author  = {Eckstein, Stephan and Pammer, Gudmund},
  journal = {arXiv preprint arXiv:2203.05005},
  title   = {Computational methods for adapted optimal transport},
  year    = {2022},
}

@ARTICLE{AccHou2024,
    AUTHOR = {Beatrice Acciaio and Songyan Hou},
     TITLE = {Convergence of adapted empirical measures on $\mathbb{R}^{d}$},
   JOURNAL = {Ann. Appl. Probab.},
  FJOURNAL = {Annals of Applied Probability},
      YEAR = {2024},
    VOLUME = {34},
    NUMBER = {5},
     PAGES = {4799-4835},
      ISSN = {1050-5164},
       DOI = {10.1214/24-AAP2082},
      SICI = {1050-5164(2024)34:5<4799:COAEMO>2.0.CO;2-M},
}

@Article{Xu2020CgG,
  author  = {Xu, Tianlin and Wenliang, Li Kevin and Munn, Michael and Acciaio, Beatrice},
  journal = {Advances in neural information processing systems},
  title   = {Cot-gan: Generating sequential data via causal optimal transport},
  year    = {2020},
  pages   = {8798--8809},
  volume  = {33},
}

@Article{Becker2019Dos,
  author    = {Becker, Sebastian and Cheridito, Patrick and Jentzen, Arnulf},
  journal   = {The Journal of Machine Learning Research},
  title     = {Deep optimal stopping},
  year      = {2019},
  number    = {1},
  pages     = {2712--2736},
  volume    = {20},
  publisher = {JMLR. org},
}

@Article{Dinh2016Deu,
  author  = {Dinh, Laurent and Sohl-Dickstein, Jascha and Bengio, Samy},
  journal = {arXiv preprint arXiv:1605.08803},
  title   = {Density estimation using real nvp},
  year    = {2016},
}

@Article{Pflug2015Dgo,
  author    = {Pflug, Georg Ch and Pichler, Alois},
  journal   = {Computational Optimization and Applications},
  title     = {Dynamic generation of scenario trees},
  year      = {2015},
  number    = {3},
  pages     = {641--668},
  volume    = {62},
  publisher = {Springer},
}

@Article{Cont2001Epo,
  author    = {Cont, Rama},
  journal   = {Quantitative finance},
  title     = {Empirical properties of asset returns: stylized facts and statistical issues},
  year      = {2001},
  number    = {2},
  pages     = {223},
  volume    = {1},
  publisher = {IOP Publishing},
}

@Article{BackhoffVeraguas2022Epi,
  author    = {Backhoff-Veraguas, Julio and Bartl, Daniel and Beiglb{\"o}ck, Mathias and Wiesel, Johannes},
  journal   = {The Annals of Applied Probability},
  title     = {Estimating processes in adapted Wasserstein distance},
  year      = {2022},
  number    = {1},
  pages     = {529--550},
  volume    = {32},
  publisher = {Institute of Mathematical Statistics},
}

@Article{Pichler2013Eor,
  author    = {Pichler, Alois},
  journal   = {SIAM Journal on Optimization},
  title     = {Evaluations of risk measures for different probability measures},
  year      = {2013},
  number    = {1},
  pages     = {530--551},
  volume    = {23},
  publisher = {SIAM},
}

@Article{Pflug2016Feo,
  author    = {Pflug, Georg Ch and Pichler, Alois},
  journal   = {SIAM Journal on Optimization},
  title     = {From empirical observations to tree models for stochastic optimization: convergence properties},
  year      = {2016},
  number    = {3},
  pages     = {1715--1740},
  volume    = {26},
  publisher = {SIAM},
}

@Article{Kolouri2019Gsw,
  author  = {Kolouri, Soheil and Nadjahi, Kimia and Simsekli, Umut and Badeau, Roland and Rohde, Gustavo},
  journal = {Advances in neural information processing systems},
  title   = {Generalized sliced wasserstein distances},
  year    = {2019},
  volume  = {32},
}

@InProceedings{Assefa2020Gsd,
  author    = {Assefa, Samuel A and Dervovic, Danial and Mahfouz, Mahmoud and Tillman, Robert E and Reddy, Prashant and Veloso, Manuela},
  booktitle = {Proceedings of the First ACM International Conference on AI in Finance},
  title     = {Generating synthetic data in finance: opportunities, challenges and pitfalls},
  year      = {2020},
  pages     = {1--8},
}

@Article{Glanzer2019Ism,
  author    = {Glanzer, Martin and Pflug, Georg Ch and Pichler, Alois},
  journal   = {Mathematical Programming},
  title     = {Incorporating statistical model error into the calculation of acceptability prices of contingent claims},
  year      = {2019},
  number    = {1},
  pages     = {499--524},
  volume    = {174},
  publisher = {Springer},
}

@Article{Nutz2021Ite,
  author  = {Nutz, Marcel},
  journal = {Lecture notes, Columbia University},
  title   = {Introduction to entropic optimal transport},
  year    = {2021},
}

@Book{Takezawa2005Itn,
  author    = {Takezawa, Kunio},
  publisher = {John Wiley \& Sons},
  title     = {Introduction to nonparametric regression},
  year      = {2005},
}

@Article{Benezet2024Lcd,
  author  = {B{\'e}n{\'e}zet, Cyril and Cheng, Ziteng and Jaimungal, Sebastian},
  journal = {arXiv preprint arXiv:2406.09375},
  title   = {Learning conditional distributions on continuous spaces},
  year    = {2024},
}

@Article{Rubinstein2002Mps,
  author    = {Rubinstein, Mark},
  journal   = {The Journal of finance},
  title     = {Markowitz's" portfolio selection": A fifty-year retrospective},
  year      = {2002},
  number    = {3},
  pages     = {1041--1045},
  volume    = {57},
  publisher = {JSTOR},
}

@InProceedings{Deshpande2019Msw,
  author    = {Deshpande, Ishan and Hu, Yuan-Ting and Sun, Ruoyu and Pyrros, Ayis and Siddiqui, Nasir and Koyejo, Sanmi and Zhao, Zhizhen and Forsyth, David and Schwing, Alexander G},
  booktitle = {Proceedings of the IEEE/CVF Conference on Computer Vision and Pattern Recognition},
  title     = {Max-sliced wasserstein distance and its use for gans},
  year      = {2019},
  pages     = {10648--10656},
}

@Article{Li2017MgT,
  author  = {Li, Chun-Liang and Chang, Wei-Cheng and Cheng, Yu and Yang, Yiming and P{\'o}czos, Barnab{\'a}s},
  journal = {Advances in neural information processing systems},
  title   = {Mmd gan: Towards deeper understanding of moment matching network},
  year    = {2017},
  volume  = {30},
}

@Article{Forsyth2022MPM,
  author    = {Forsyth, Peter A and Vetzal, Kenneth R},
  journal   = {Applied Mathematical Finance},
  title     = {Multi-Period Mean Expected-Shortfall Strategies:‘Cut Your Losses and Ride Your Gains’},
  year      = {2022},
  pages     = {1--37},
  publisher = {Taylor \& Francis},
}

@Book{Pflug2014Mso,
  author    = {Pflug, Georg Ch and Pichler, Alois},
  publisher = {Springer},
  title     = {Multistage stochastic optimization},
  year      = {2014},
  volume    = {1104},
}

@InProceedings{Kidger2021Nsa,
  author       = {Kidger, Patrick and Foster, James and Li, Xuechen and Lyons, Terry J},
  booktitle    = {International conference on machine learning},
  title        = {Neural sdes as infinite-dimensional gans},
  year         = {2021},
  organization = {PMLR},
  pages        = {5453--5463},
}

@Article{BionNadal2019OaW,
  author    = {Bion--Nadal, Jocelyne and Talay, Denis},
  journal   = {The Annals of Applied Probability},
  title     = {On a Wasserstein-type distance between solutions to stochastic differential equations},
  year      = {2019},
  number    = {3},
  pages     = {1609--1639},
  volume    = {29},
  publisher = {Institute of Mathematical Statistics},
}

@PhdThesis{Gigli2008Otg,
  author = {Gigli, Nicola},
  school = {Thesis (Ph. D.)--Scuola Normale Superiore},
  title  = {On the geometry of the space of probability measures in Rn endowed with the quadratic optimal transport distance},
  year   = {2008},
}

@Article{Li2000Odp,
  author    = {Li, Duan and Ng, Wan-Lung},
  journal   = {Mathematical finance},
  title     = {Optimal dynamic portfolio selection: Multiperiod mean-variance formulation},
  year      = {2000},
  number    = {3},
  pages     = {387--406},
  volume    = {10},
  publisher = {Wiley Online Library},
}

@InCollection{Merton1975Oca,
  author    = {Merton, Robert C},
  booktitle = {Stochastic optimization models in finance},
  publisher = {Elsevier},
  title     = {Optimum consumption and portfolio rules in a continuous-time model},
  year      = {1975},
  pages     = {621--661},
}

@Article{Lou2024PGg,
  author  = {Lou, Hang and Li, Siran and Ni, Hao},
  journal = {Advances in Neural Information Processing Systems},
  title   = {PCF-GAN: generating sequential data via the characteristic function of measures on the path space},
  year    = {2024},
  volume  = {36},
}

@Article{Flamary2021PPO,
  author  = {R{\'e}mi Flamary and Nicolas Courty and Alexandre Gramfort and Mokhtar Z. Alaya and Aur{\'e}lie Boisbunon and Stanislas Chambon and Laetitia Chapel and Adrien Corenflos and Kilian Fatras and Nemo Fournier and L{\'e}o Gautheron and Nathalie T.H. Gayraud and Hicham Janati and Alain Rakotomamonjy and Ievgen Redko and Antoine Rolet and Antony Schutz and Vivien Seguy and Danica J. Sutherland and Romain Tavenard and Alexander Tong and Titouan Vayer},
  journal = {Journal of Machine Learning Research},
  title   = {POT: Python Optimal Transport},
  year    = {2021},
  number  = {78},
  pages   = {1-8},
  volume  = {22},
}

@Article{Chevyrev2018Smt,
  author  = {Chevyrev, Ilya and Oberhauser, Harald},
  journal = {arXiv preprint arXiv:1810.10971},
  title   = {Signature moments to characterize laws of stochastic processes},
  year    = {2018},
}

@InProceedings{Ni2021SWG,
  author    = {Ni, Hao and Szpruch, Lukasz and Sabate-Vidales, Marc and Xiao, Baoren and Wiese, Magnus and Liao, Shujian},
  booktitle = {Proceedings of the Second ACM International Conference on AI in Finance},
  title     = {Sig-Wasserstein GANs for time series generation},
  year      = {2021},
  pages     = {1--8},
}

@Article{Nietert2022Sra,
  author  = {Nietert, Sloan and Goldfeld, Ziv and Sadhu, Ritwik and Kato, Kengo},
  journal = {Advances in Neural Information Processing Systems},
  title   = {Statistical, robustness, and computational guarantees for sliced wasserstein distances},
  year    = {2022},
  pages   = {28179--28193},
  volume  = {35},
}

@Book{Foellmer2011Sfa,
  author    = {F{\"o}llmer, Hans and Schied, Alexander},
  publisher = {Walter de Gruyter},
  title     = {Stochastic finance: an introduction in discrete time},
  year      = {2011},
}

@Article{Cont2022TgL,
  author  = {Cont, Rama and Cucuringu, Mihai and Xu, Renyuan and Zhang, Chao},
  journal = {arXiv preprint arXiv:2203.01664},
  title   = {Tail-gan: Learning to simulate tail risk scenarios},
  year    = {2022},
}

@Article{Rezende2018Tv,
  author  = {Rezende, Danilo Jimenez and Viola, Fabio},
  journal = {arXiv preprint arXiv:1810.00597},
  title   = {Taming vaes},
  year    = {2018},
}

@Article{Pichler2022TnS,
  author    = {Pichler, Alois and Weinhardt, Michael},
  journal   = {Computational Management Science},
  title     = {The nested Sinkhorn divergence to learn the nested distance},
  year      = {2022},
  number    = {2},
  pages     = {269--293},
  volume    = {19},
  publisher = {Springer},
}

@Article{Black1973Tpo,
  author    = {Black, Fischer and Scholes, Myron},
  journal   = {Journal of political economy},
  title     = {The pricing of options and corporate liabilities},
  year      = {1973},
  number    = {3},
  pages     = {637--654},
  volume    = {81},
  publisher = {The University of Chicago Press},
}

@Article{Bailey2016Tpo,
  author  = {Bailey, David H and Borwein, Jonathan and Lopez de Prado, Marcos and Zhu, Qiji Jim},
  journal = {Journal of Computational Finance, forthcoming},
  title   = {The probability of backtest overfitting},
  year    = {2016},
}

@Article{Rueschendorf1985TWd,
  author    = {R{\"u}schendorf, Ludger},
  journal   = {Probability Theory and Related Fields},
  title     = {The Wasserstein distance and approximation theorems},
  year      = {1985},
  number    = {1},
  pages     = {117--129},
  volume    = {70},
  publisher = {Springer},
}

@Article{Guyon2023Vim,
  author    = {Guyon, Julien and Lekeufack, Jordan},
  journal   = {Quantitative Finance},
  title     = {Volatility is (mostly) path-dependent},
  year      = {2023},
  number    = {9},
  pages     = {1221--1258},
  volume    = {23},
  publisher = {Taylor \& Francis},
}

@Article{Stanczuk2021WGw,
  author  = {Stanczuk, Jan and Etmann, Christian and Kreusser, Lisa Maria and Sch{\"o}nlieb, Carola-Bibiane},
  journal = {arXiv preprint arXiv:2103.01678},
  title   = {Wasserstein GANs work because they fail (to approximate the Wasserstein distance)},
  year    = {2021},
}

@Article{Hou2024Cot,
  author  = {Hou, Songyan},
  journal = {arXiv preprint arXiv:2401.14883},
  title   = {Convergence of the Adapted Smoothed Empirical Measures},
  year    = {2024},
}

@Book{Rachev2013Tmo,
  author    = {Rachev, Svetlozar T and Klebanov, Lev B and Stoyanov, Stoyan V and Fabozzi, Frank},
  publisher = {Springer},
  title     = {The methods of distances in the theory of probability and statistics},
  year      = {2013},
  volume    = {10},
}

@Article{Goodfellow2014GAN,
  author    = {Ian Goodfellow and Jean Pouget-Abadie and Mehdi Mirza and Bing Xu and David Warde-Farley and Sherjil Ozair and Aaron Courville and Yoshua Bengio},
  journal   = {Advances in Neural Information Processing Systems},
  title     = {Generative Adversarial Nets},
  year      = {2014},
  volume    = {27},
  editor    = {Z Ghahramani and M Welling and C Cortes and N Lawrence and K Q Weinberger},
  publisher = {Curran Associates, Inc.},
}

@Article{Koshiyama2021Gan,
  author  = {Adriano Koshiyama and Nick Firoozye and Philip Treleaven},
  journal = {Quantitative Finance},
  title   = {Generative adversarial networks for financial trading strategies fine-tuning and combination},
  year    = {2021},
  month   = {5},
  pages   = {797-813},
  volume  = {21},
  issue   = {5},
}

@Article{Hamdouche2023GMf,
  author  = {Mohamed Hamdouche and Pierre Henry-Labordere and Huyen Pham},
  journal = {SSRN Electronic Journal},
  title   = {Generative Modeling for Time Series Via Schrödinger Bridge},
  year    = {2023},
  issn    = {1556-5068},
}

@Article{Takahashi2019Mft,
  author  = {Shuntaro Takahashi and Yu Chen and Kumiko Tanaka-Ishii},
  journal = {Physica A: Statistical Mechanics and its Applications},
  title   = {Modeling financial time-series with generative adversarial networks},
  year    = {2019},
  month   = {8},
  pages   = {121261},
  volume  = {527},
}

@Article{Issa2024Nat,
  author  = {Issa, Zacharia and Horvath, Blanka and Lemercier, Maud and Salvi, Cristopher},
  journal = {Advances in Neural Information Processing Systems},
  title   = {Non-adversarial training of Neural SDEs with signature kernel scores},
  year    = {2024},
  volume  = {36},
}

@Misc{Kidger2021Ond,
  author = {P. Kidger},
  title  = {On neural differential equations (PhD Thesis)},
  year   = {2021},
  school = {University of Oxford},
}

@Article{Esteban2017RvM,
  author  = {Cristobal Esteban and Stephanie Hyland and Gunnar R\"atsch},
  journal = {ArXiv},
  title   = {Real-valued (Medical) Time Series Generation with Recurrent Conditional GANs},
  year    = {2017},
  month   = {8},
  volume  = {1706.02633},
}

@Article{Fu2022Sft,
  author  = {Weilong Fu and Ali Hirsa and J\"org Osterrieder},
  journal = {ArXiv},
  title   = {Simulating financial time series using attention},
  year    = {2022},
  month   = {7},
  volume  = {2207.00493},
}

@InProceedings{Coletta2021Trm,
  author    = {Coletta, Andrea and Prata, Matteo and Conti, Michele and Mercanti, Emanuele and Bartolini, Novella and Moulin, Aymeric and Vyetrenko, Svitlana and Balch, Tucker},
  booktitle = {Proceedings of the Second ACM International Conference on AI in Finance},
  title     = {Towards realistic market simulations: a generative adversarial networks approach},
  year      = {2021},
  pages     = {1--9},
}

@Article{Buehler2020Add,
  author  = {Hans B{\"u}hler and Blanka Horvath and Terry Lyons and Imanol Perez Arribas and Ben Wood},
  journal = {ERN: Neural Networks \& Related Topics (Topic)},
  title   = {A data-driven market simulator for small data environments},
  year    = {2020},
}

@Article{Wiese2021MAS,
  author  = {Magnus Wiese and Ben Wood and Alexandre Pachoud and Ralf Korn and Hans Buehler and Murray Phillip and Lianjun Bai},
  journal = {SSRN Electronic Journal},
  title   = {Multi-Asset Spot and Option Market Simulation},
  year    = {2021},
}

@Article{Wiese2020QGd,
  author  = {Magnus Wiese and Robert Knobloch and Ralf Korn and Peter Kretschmer},
  journal = {Quantitative Finance},
  title   = {Quant GANs: deep generation of financial time series},
  year    = {2020},
  month   = {9},
  pages   = {1419-1440},
  volume  = {20},
  issue   = {9},
}

@Article{Kondratyev2019TMG,
  author  = {Alex Kondratyev and Christian Schwarz},
  journal = {Econometrics: Econometric \& Statistical Methods - Special Topics eJournal},
  title   = {The Market Generator},
  year    = {2019},
}

@Article{Yoon2019TsG,
  author    = {Jinsung Yoon and Daniel Jarrett and Mihaela van der Schaar},
  journal   = {Advances in Neural Information Processing Systems},
  title     = {Time-series Generative Adversarial Networks},
  year      = {2019},
  volume    = {32},
  editor    = {H Wallach and H Larochelle and A Beygelzimer and F d Alché-Buc and E Fox and R Garnett},
  publisher = {Curran Associates, Inc.},
}

@InProceedings{Arjovsky2017WGA,
  author       = {Arjovsky, Martin and Chintala, Soumith and Bottou, L{\'e}on},
  booktitle    = {Proceedings of the 34th International Conference on Machine Learning},
  title        = {{W}asserstein Generative Adversarial Networks},
  year         = {2017},
  editor       = {Precup, Doina and Teh, Yee Whye},
  month        = {8},
  organization = {PMLR},
  pages        = {214--223},
  volume       = {70},
}

@article{biagini2024universal,
  title={Universal randomised signatures for generative time series modelling},
  author={Biagini, Francesca and Gonon, Lukas and Walter, Niklas},
  journal={arXiv preprint arXiv:2406.10214},
  year={2024}
}

@article{chung2024generative,
  title={Generative model for financial time series trained with MMD using a signature kernel},
  author={Chung I, Lu and Sester, Julian},
  journal={arXiv preprint arXiv:2407.19848},
  year={2024}
}

@article{eckerli2021generative,
  title={Generative adversarial networks in finance: an overview},
  author={Eckerli, Florian and Osterrieder, Joerg},
  journal={arXiv preprint arXiv:2106.06364},
  year={2021}
}

@inproceedings{assefa2020generating,
  title={Generating synthetic data in finance: opportunities, challenges and pitfalls},
  author={Assefa, Samuel A and Dervovic, Danial and Mahfouz, Mahmoud and Tillman, Robert E and Reddy, Prashant and Veloso, Manuela},
  booktitle={Proceedings of the First ACM International Conference on AI in Finance},
  pages={1--8},
  year={2020}
}

@article{de2019enriching,
  title={Enriching financial datasets with generative adversarial networks},
  author={de Meer Pardo, Fernando},
  journal={MS thesis, Delft University of Technology, The Netherlands},
  year={2019}
}

@article{efimov2020using,
  title={Using generative adversarial networks to synthesize artificial financial datasets},
  author={Efimov, Dmitry and Xu, Di and Kong, Luyang and Nefedov, Alexey and Anandakrishnan, Archana},
  journal={arXiv preprint arXiv:2002.02271},
  year={2020}
}

@article{liao2024sig,
  title={Sig-Wasserstein GANs for conditional time series generation},
  author={Liao, Shujian and Ni, Hao and Sabate-Vidales, Marc and Szpruch, Lukasz and Wiese, Magnus and Xiao, Baoren},
  journal={Mathematical Finance},
  volume={34},
  number={2},
  pages={622--670},
  year={2024},
  publisher={Wiley Online Library}
}

@article{dilokthanakul2016deep,
  title={Deep unsupervised clustering with gaussian mixture variational autoencoders},
  author={Dilokthanakul, Nat and Mediano, Pedro AM and Garnelo, Marta and Lee, Matthew CH and Salimbeni, Hugh and Arulkumaran, Kai and Shanahan, Murray},
  journal={arXiv preprint arXiv:1611.02648},
  year={2016}
}

@inproceedings{tomczak2018vae,
  title={VAE with a VampPrior},
  author={Tomczak, Jakub and Welling, Max},
  booktitle={International conference on artificial intelligence and statistics},
  pages={1214--1223},
  year={2018},
  organization={PMLR}
}

@article{bishop1998gtm,
  title={GTM: The generative topographic mapping},
  author={Bishop, Christopher M and Svens{\'e}n, Markus and Williams, Christopher KI},
  journal={Neural computation},
  volume={10},
  number={1},
  pages={215--234},
  year={1998},
  publisher={MIT Press}
}

@article{chen2016variational,
  title={Variational lossy autoencoder},
  author={Chen, Xi and Kingma, Diederik P and Salimans, Tim and Duan, Yan and Dhariwal, Prafulla and Schulman, John and Sutskever, Ilya and Abbeel, Pieter},
  journal={arXiv preprint arXiv:1611.02731},
  year={2016}
}

@article{vuletic2024fin,
  title={Fin-gan: Forecasting and classifying financial time series via generative adversarial networks},
  author={Vuleti{\'c}, Milena and Prenzel, Felix and Cucuringu, Mihai},
  journal={Quantitative Finance},
  volume={24},
  number={2},
  pages={175--199},
  year={2024},
  publisher={Taylor \& Francis}
}

@article{rizzato2023generative,
  title={Generative Adversarial Networks applied to synthetic financial scenarios generation},
  author={Rizzato, Matteo and Wallart, Julien and Geissler, Christophe and Morizet, Nicolas and Boumlaik, Noureddine},
  journal={Physica A: Statistical Mechanics and its Applications},
  volume={623},
  pages={128899},
  year={2023},
  publisher={Elsevier}
}

@article{ericson2024deep,
  title={Deep Generative Modeling for Financial Time Series with Application in VaR: A Comparative Review},
  author={Ericson, Lars and Zhu, Xuejun and Han, Xusi and Fu, Rao and Li, Shuang and Guo, Steve and Hu, Ping},
  journal={arXiv preprint arXiv:2401.10370},
  year={2024}
}

@article{lu2023machine,
  title={Machine learning for synthetic data generation: a review},
  author={Lu, Yingzhou and Shen, Minjie and Wang, Huazheng and Wang, Xiao and van Rechem, Capucine and Wei, Wenqi},
  journal={arXiv preprint arXiv:2302.04062},
  year={2023}
}

@article{iglesias2023data,
  title={Data augmentation techniques in time series domain: a survey and taxonomy},
  author={Iglesias, Guillermo and Talavera, Edgar and Gonz{\'a}lez-Prieto, {\'A}ngel and Mozo, Alberto and G{\'o}mez-Canaval, Sandra},
  journal={Neural Computing and Applications},
  volume={35},
  number={14},
  pages={10123--10145},
  year={2023},
  publisher={Springer}
}

@article{cont2023limit,
  title={Limit Order Book Simulation with Generative Adversarial Networks},
  author={Cont, Rama and Cucuringu, Mihai and Kochems, Jonathan and Prenzel, Felix},
  journal={Available at SSRN 4512356},
  year={2023}
}

@article{hultin2023generative,
  title={A generative model of a limit order book using recurrent neural networks},
  author={Hultin, Hanna and Hult, Henrik and Proutiere, Alexandre and Samama, Samuel and Tarighati, Ala},
  journal={Quantitative Finance},
  volume={23},
  number={6},
  pages={931--958},
  year={2023},
  publisher={Taylor \& Francis}
}

@inproceedings{nagy2023generative,
  title={Generative ai for end-to-end limit order book modelling: A token-level autoregressive generative model of message flow using a deep state space network},
  author={Nagy, Peer and Frey, Sascha and Sapora, Silvia and Li, Kang and Calinescu, Anisoara and Zohren, Stefan and Foerster, Jakob},
  booktitle={Proceedings of the Fourth ACM International Conference on AI in Finance},
  pages={91--99},
  year={2023}
}

@inproceedings{li2020generating,
  title={Generating realistic stock market order streams},
  author={Li, Junyi and Wang, Xintong and Lin, Yaoyang and Sinha, Arunesh and Wellman, Michael},
  booktitle={Proceedings of the AAAI Conference on Artificial Intelligence},
  volume={34},
  number={01},
  pages={727--734},
  year={2020}
}

@article{ozyar2021learning,
  title={Learning the Limit Order Book: a comprehensive mix between stochastic and machine learning models for generation and prediction},
  author={{\"O}zyar, Muhammed Imran},
  year={2021}
}

@phdthesis{hultin2021generative,
  title={Generative models of limit order books},
  author={Hultin, Hanna},
  year={2021},
  school={KTH Royal Institute of Technology}
}

@inproceedings{coletta2023conditional,
  title={Conditional generators for limit order book environments: Explainability, challenges, and robustness},
  author={Coletta, Andrea and Jerome, Joseph and Savani, Rahul and Vyetrenko, Svitlana},
  booktitle={Proceedings of the Fourth ACM International Conference on AI in Finance},
  pages={27--35},
  year={2023}
}

@inproceedings{thanh2020catastrophic,
  title={Catastrophic forgetting and mode collapse in GANs},
  author={Thanh-Tung, Hoang and Tran, Truyen},
  booktitle={2020 international joint conference on neural networks (ijcnn)},
  pages={1--10},
  year={2020},
  organization={IEEE}
}

@article{karras2020training,
  title={Training generative adversarial networks with limited data},
  author={Karras, Tero and Aittala, Miika and Hellsten, Janne and Laine, Samuli and Lehtinen, Jaakko and Aila, Timo},
  journal={Advances in neural information processing systems},
  volume={33},
  pages={12104--12114},
  year={2020}
}

@article{chung2015recurrent,
  title={A recurrent latent variable model for sequential data},
  author={Chung, Junyoung and Kastner, Kyle and Dinh, Laurent and Goel, Kratarth and Courville, Aaron C and Bengio, Yoshua},
  journal={Advances in neural information processing systems},
  volume={28},
  year={2015}
}

@article{dinh2016density,
  title={Density estimation using real nvp},
  author={Dinh, Laurent and Sohl-Dickstein, Jascha and Bengio, Samy},
  journal={arXiv preprint arXiv:1605.08803},
  year={2016}
}

@article{gatopoulos2020super,
  title={Super-resolution variational auto-encoders},
  author={Gatopoulos, Ioannis and Stol, Maarten and Tomczak, Jakub M},
  journal={arXiv preprint arXiv:2006.05218},
  year={2020}
}

@article{child2020very,
  title={Very deep vaes generalize autoregressive models and can outperform them on images},
  author={Child, Rewon},
  journal={arXiv preprint arXiv:2011.10650},
  year={2020}
}

@article{girin2020dynamical,
  title={Dynamical variational autoencoders: A comprehensive review},
  author={Girin, Laurent and Leglaive, Simon and Bie, Xiaoyu and Diard, Julien and Hueber, Thomas and Alameda-Pineda, Xavier},
  journal={arXiv preprint arXiv:2008.12595},
  year={2020}
}

@article{bowman2015generating,
  title={Generating sentences from a continuous space},
  author={Bowman, Samuel R and Vilnis, Luke and Vinyals, Oriol and Dai, Andrew M and Jozefowicz, Rafal and Bengio, Samy},
  journal={arXiv preprint arXiv:1511.06349},
  year={2015}
}

@book{cinelli2021variational,
  title={Variational methods for machine learning with applications to deep networks},
  author={Cinelli, Lucas Pinheiro and Marins, Matheus Ara{\'u}jo and Da Silva, Eduardo Antonio Barros and Netto, S{\'e}rgio Lima},
  volume={15},
  year={2021},
  publisher={Springer}
}

@article{desai2021timevae,
  title={Timevae: A variational auto-encoder for multivariate time series generation},
  author={Desai, Abhyuday and Freeman, Cynthia and Wang, Zuhui and Beaver, Ian},
  journal={arXiv preprint arXiv:2111.08095},
  year={2021}
}

@article{cai2023hybrid,
  title={Hybrid variational autoencoder for time series forecasting},
  author={Cai, Borui and Yang, Shuiqiao and Gao, Longxiang and Xiang, Yong},
  journal={Knowledge-Based Systems},
  volume={281},
  pages={111079},
  year={2023},
  publisher={Elsevier}
}

@article{liu2022time,
  title={Time-Transformer AAE: Connecting Temporal Convolutional Networks and Transformer for Time Series Generation},
  author={Liu, Yuansan and Wijewickrema, Sudanthi and Li, Ang and Bailey, James},
  year={2022}
}

@inproceedings{huang2024generative,
  title={Generative Learning for Financial Time Series with Irregular and Scale-Invariant Patterns},
  author={Huang, Hongbin and Chen, Minghua and Qiao, Xiao},
  booktitle={The Twelfth International Conference on Learning Representations},
  year={2024}
}

@inproceedings{yang2021causal,
  title={Causal attention for vision-language tasks},
  author={Yang, Xu and Zhang, Hanwang and Qi, Guojun and Cai, Jianfei},
  booktitle={Proceedings of the IEEE/CVF conference on computer vision and pattern recognition},
  pages={9847--9857},
  year={2021}
}

@inproceedings{bhowalvariational,
  title={Why do Variational Autoencoders Really Promote Disentanglement?},
  author={Bhowal, Pratik and Soni, Achint and Rambhatla, Sirisha},
  booktitle={Forty-first International Conference on Machine Learning},
    year={2024}
}

@article{chen2018isolating,
  title={Isolating sources of disentanglement in variational autoencoders},
  author={Chen, Ricky TQ and Li, Xuechen and Grosse, Roger B and Duvenaud, David K},
  journal={Advances in neural information processing systems},
  volume={31},
  year={2018}
}

@inproceedings{mathieu2019disentangling,
  title={Disentangling disentanglement in variational autoencoders},
  author={Mathieu, Emile and Rainforth, Tom and Siddharth, Nana and Teh, Yee Whye},
  booktitle={International conference on machine learning},
  pages={4402--4412},
  year={2019},
  organization={PMLR}
}

@article{burgess2018understanding,
  title={Understanding disentangling in $beta$-VAE},
  author={Burgess, Christopher P and Higgins, Irina and Pal, Arka and Matthey, Loic and Watters, Nick and Desjardins, Guillaume and Lerchner, Alexander},
  journal={arXiv preprint arXiv:1804.03599},
  year={2018}
}

@article{schwarz2024interpretable,
  title={Interpretable GenAI: Synthetic Financial Time Series Generation with Probabilistic LSTM},
  author={Schwarz, Christian},
  journal={Available at SSRN 4877007},
  year={2024}
}
\end{document}